\documentclass{article} 
\usepackage{iclr2022_conference,times}

\usepackage[utf8]{inputenc} 
\usepackage[T1]{fontenc}    
\usepackage{hyperref}       
\usepackage{url}            
\usepackage{booktabs}       
\usepackage{amsfonts}       
\usepackage{nicefrac}       
\usepackage{microtype}      

\usepackage{booktabs}
\usepackage{multirow}

\usepackage{amssymb}
\usepackage{amsmath}
\usepackage{amsthm}
\usepackage{bbold}
\usepackage{bm}
\usepackage{siunitx}

\usepackage{natbib}
\usepackage{tikz}
\usepackage{tkz-euclide}
\usepackage{chngcntr}
\usepackage{verbatim}
\usepackage{mathtools}
\usepackage{esvect}
\usepackage{subfiles}
\usepackage{xcolor}
\usepackage{xspace}
\usepackage{xparse}
\usepackage{algorithm}
\usepackage{algorithm}
\usepackage{algorithmicx}
\usepackage{algpseudocode}
\usepackage{thmtools,thm-restate}
\usepackage{graphicx}
\usepackage{sidecap}
\usepackage{subcaption}
\usepackage{wrapfig}

\usepackage{paracol}
\usepackage[compact]{titlesec} 
\globalcounter{algocf}

\usepackage{enumitem}
\usepackage[flushleft]{threeparttable}
\usepackage{scrextend} 

\usepackage{cleveref}
\usepackage[export]{adjustbox} 
\usepackage{mathrsfs} 
\usetikzlibrary{arrows}

\DeclarePairedDelimiterX{\inp}[2]{\langle}{\rangle}{#1, #2}
\DeclarePairedDelimiterX{\abs}[1]{\lvert}{\rvert}{#1}
\DeclarePairedDelimiterX{\roundup}[1]{\lceil}{\rceil}{#1}
\DeclarePairedDelimiterX{\norm}[1]{\lVert}{\rVert}{#1}
\DeclarePairedDelimiterX{\cbr}[1]{\{}{\}}{#1} 
\DeclarePairedDelimiterX{\rbr}[1]{(}{)}{#1} 
\DeclarePairedDelimiterX{\sbr}[1]{[}{]}{#1} 

\providecommand{\tsum}{\textstyle\sum} 
\providecommand{\N}{\mathbb{N}} 

\DeclareMathOperator{\expect}{\mathbb{E}}
\DeclareMathOperator{\E}{\expect}

\DeclareMathOperator{\sgn}{sign}
\makeatletter
\def\sign{\@ifnextchar*{\@sgnargscaled}{\@ifnextchar[{\sgnargscaleas}{\@ifnextchar{\bgroup}{\@sgnarg}{\sgn} }}}
\def\@sgnarg#1{\sgn\rbr{#1}}
\def\@sgnargscaled#1{\sgn\rbr*{#1}}
\def\@sgnargscaleas[#1]#2{\sgn\rbr[#1]{#2}}
\makeatother

\DeclareMathOperator*{\argmin}{arg\!min}
\DeclareMathOperator*{\argmax}{arg\!max}

\providecommand{\0}{\bm{0}}

\providecommand{\ee}{\bm{e}}

\let\ggg\gg
\renewcommand{\gg}{\bm{g}}

\providecommand{\mm}{\bm{m}}

\renewcommand{\vv}{\bm{v}}

\providecommand{\xx}{\bm{x}}
\providecommand{\yy}{\bm{y}}
\providecommand{\zz}{\bm{z}}

\providecommand{\xiv}{\boldsymbol{\xi}}
\providecommand{\muv}{\boldsymbol{\mu}}

\providecommand{\cB}{\mathcal{B}}

\providecommand{\cE}{\mathcal{E}}

\providecommand{\cG}{\mathcal{G}}

\providecommand{\cO}{\mathcal{O}}

\providecommand{\cS}{\mathcal{S}}

\newtheorem{theorem}{Theorem}

\newtheorem{lemma}{Lemma}
\newtheorem{remark}[lemma]{Remark}

\newtheorem{definition}{Definition}

\newcommand{\ignore}[1]{}

\definecolor{color1}{RGB}{228,26,28}
\definecolor{color2}{RGB}{55,126,184}
\definecolor{color3}{RGB}{77,175,74}
\definecolor{color4}{RGB}{152,78,163}
\definecolor{color5}{RGB}{255,127,0}

\definecolor{mygreen}{RGB}{77,175,74}
\definecolor{myorange}{RGB}{224,155,21}
\definecolor{myblue}{RGB}{55,126,184}
\definecolor{skyblue}{RGB}{117,187,253}
\definecolor{myred}{RGB}{228,26,28}

\usepackage{pifont}
\makeatletter
\newcommand{\myitem}[1]{%
  \item[\textbf{(#1)}]\protected@edef\@currentlabel{#1}%
}
\makeatother

\usetikzlibrary{fit,calc}

\colorlet{client}{red!40}

\setlist{leftmargin=5.5mm}
\DeclareCaptionFormat{myformat}{\fontsize{8.5}{10}\selectfont#1#2#3}

\hypersetup{
    colorlinks=true,
    citecolor=mygreen,
}

\newcommand{\krum}{\textsc{Krum}\xspace}
\newcommand{\resample}{\textsc{Bucketing}\xspace}
\newcommand{\rfa}{\textsc{RFA}\xspace}
\newcommand{\rsa}{\textsc{Rsa}\xspace}
\newcommand{\cm}{\textsc{CM}\xspace}
\newcommand{\tm}{\textsc{TM}\xspace}
\newcommand{\cclip}{\textsc{CClip}\xspace}

\newcommand{\cI}{\mathcal{I}}

\newcommand{\iid}{{iid}\xspace}
\newcommand{\noniid}{{non-iid}\xspace}

\newcommand{\aggr}{\textsc{Aggr}\xspace}
\newcommand{\ragg}{\textsc{ARAgg}\xspace}
\newcommand{\alg}{\textsc{Alg}\xspace}

\algrenewcommand\algorithmicrequire{\textbf{Input:}}
\makeatletter
\ifthenelse{\equal{\ALG@noend}{t}}%
{\algtext*{EndOn}}
{}%
\makeatother

\newtheorem*{rep@theorem}{\protect\rep@title}
\newcommand{\newreptheorem}[2]{%
\newenvironment{rep#1}[1]{%
 \def\rep@title{#2 \ref{##1}}%
 \begin{rep@theorem}}%
 {\end{rep@theorem}}}
\newreptheorem{lemma}{Lemma'}
\newreptheorem{definition}{Definition'}
\newreptheorem{proposition}{Proposition'}
\newreptheorem{theorem}{Theorem'}

\title{Byzantine-Robust Learning on Heterogeneous Datasets via Bucketing}

\iclrfinalcopy
\author{Sai Praneeth Karimireddy\thanks{equal contribution.} \\
\texttt{sai.karimireddy@epfl.ch}\\
\And 
Lie He$^*$ \\
\texttt{lie.he@epfl.ch}\\
\And
Martin Jaggi \\
\texttt{martin.jaggi@epfl.ch}\\
}

\begin{document}

\maketitle
%

\begin{abstract}
    In Byzantine robust distributed or federated learning, a central server wants to train a machine learning model over data distributed across multiple workers. However, a fraction of these workers may deviate from the prescribed algorithm and send arbitrary messages. While this problem has received significant attention recently, most current defenses assume that the workers have identical data. For realistic cases when the data across workers are heterogeneous (\noniid), we design new attacks which circumvent current defenses, leading to significant loss of performance. We then propose a simple bucketing scheme that adapts existing robust algorithms to heterogeneous datasets at a negligible computational cost. We also theoretically and experimentally validate our approach, showing that combining bucketing with existing robust algorithms is effective against challenging attacks. Our work is the first to establish guaranteed convergence for the non-iid Byzantine robust problem under realistic assumptions.
\end{abstract}

\section{Introduction}

Distributed or federated machine learning, where the data is distributed across multiple workers, has become an increasingly important learning paradigm both due to growing sizes of datasets, as well as data privacy concerns. In such a setting, the workers collaborate to train a single model without directly transmitting their training data \citep{mcmahan2016communicationefficient,bonawitz2019towards,kairouz2019federated}. However, by decentralizing the training across a vast number of workers we potentially open ourselves to new security threats.
Due to the presence of agents in the network which are actively malicious, or simply due to system and network failures, some workers may disobey the protocols and send arbitrary messages%
; such workers are also known as \textit{Byzantine} workers \citep{lamport2019byzantine}.
Byzantine robust optimization algorithms attempt to combine the updates received from the workers using robust aggregation rules and ensure that the training is not impacted by the presence of a small number of malicious workers.

While this problem has received significant recent attention due to its importance, \citep{blanchard2017machine,yin2018byzantinerobust,alistarh2018byzantine,karimireddy2020learning}, most of the current approaches assume that the data present on each different worker has identical distribution. This assumption is very unrealistic in practice and heterogeneity is inherent in distributed and federated learning~\citep{kairouz2019federated}. In this work, we show that existing Byzantine aggregation rules catastrophically fail with very simple attacks (or sometimes even with no attacks) in realistic settings. We carefully examine the causes of these failures, and propose a simple solution which provably solves the Byzantine resilient optimization problem under heterogeneous workers.

Concretely, our contributions in this work are summarized below
\begin{itemize}[nolistsep,itemsep=1mm]
    \item We show that when the data across workers is heterogeneous, existing aggregation rules fail to converge, even when no Byzantine adversaries are present. We also propose a simple new attack, \textit{mimic}, which explicitly takes advantage of data heterogeneity and circumvents median-based defenses. Together, these highlight the fragility of existing methods in real world settings.
    \item We then propose a simple fix - a new bucketing step which can be used before any existing aggregation rule. We introduce a formal notion of a robust aggregator (\ragg) and prove that existing methods like \krum , coordinate-wise median (\cm), and geometric median aka robust federated averaging (\rfa)---though insufficient on their own---become provably robust aggregators when augmented with our bucketing.%
    \item We combine our notion of robust aggregator (\ragg) with worker momentum to obtain optimal rates for Byzantine robust optimization with matching lower bounds. Unfortunately, our lower bounds imply that convergence to an exact optimum may not be possible due to heterogeneity. We then circumvent this lower bound and show that when heterogeneity is mild (or when the model is overparameterized), we can in fact converge to an exact optimum. This is the first result establishing convergence to the optimum for heterogeneous Byzantine robust optimization.
    \item Finally, we evaluate the effect of the proposed techniques (bucketing and worker momentum) against known and new attacks showcasing drastic improvement on realistic heterogeneously distributed datasets.
\end{itemize}

\textbf{Setup and notations.}
Suppose that of the total $n$ workers, the set of good workers is denoted by $\cG \subseteq \{1,\dots,n\}$. Our objective is to minimize
\begin{equation}
    \label{eqn:loss-def}\textstyle
    f(\xx) := \frac{1}{\abs{\cG}} \sum_{i \in \cG} \!\big\{ f_i(\xx) := \E_{\xiv_i}[ F_i(\xx; \xiv_i)] \big\}
\end{equation}
where $f_i$ is the loss function on worker $i$ defined over its own (heterogeneous) data distribution~$\xiv_i$.
%
%
%
The (stochastic) gradient computed by a good worker $i \in \cG$ over minibatch $\xiv_i$ is given as $\gg_i(\xx, \xiv_i) := \nabla F_i(\xx; \xiv_i)$. The noise in every stochastic gradient is independent, unbiased with $\E_{\xiv_i} [\gg_i(\xx, \xiv_i)] = \nabla f_i(\xx)$, and has bounded variance $\E_{\xiv_i} \norm{ \gg_i(\xx, \xiv_i) - \nabla f_i(\xx)}^2 \leq \sigma^2$. Further, we assume that the data heterogeneity across the workers can be bounded as
\begin{equation*}
    \E_{j \sim \cG}\norm{\nabla f_j(\xx) - \nabla f(\xx)}^2 \leq \zeta^2\,, \quad \forall \xx\,.
\end{equation*}
We write $\gg_i^t$ or simply $\gg_i$ instead of $\gg_i(\xx^t, \xiv_i^t)$ when there is no ambiguity.
%
%

\textbf{Byzantine attack model.}
The set of Byzantine workers $\cB \subset [n]$ is fixed over time, with the remaining workers $\cG$ being good, i.e. $[n] = \cB \uplus \cG$. We write $\delta$ for the fraction of Byzantine workers, $\abs{\cB} =: q \leq \delta n$. The Byzantine workers can collude and deviate arbitrarily from our protocol, sending any update to the server.

Our modeling assumes that the practitioner picks a value of $\delta \in [0,0.5)$. This $\delta$ reflects the level of robustness required. A choice of a large $\delta$ (say near 0.5) would mean that the system is very robust and can tolerate a large fraction of attackers, but the algorithm becomes much more conservative and slow. On the flip side, if the practitioner knows that the the number of Byzantine agents are going to be few, they can pick a small $\delta$ (say 0.05--0.1) ensuring some robustness with almost no impact on convergence. The choice of $\delta$ can also be formulated as how expensive do we want to make an attack? To carry out a succesful attack the attacker would need to control $\delta$ fraction of all workers. We recommend implementations claiming robustness be transparent about their choice of $\delta$.



\section{Related work} %
\label{sec:literature_review}

\textbf{IID defenses.} There has been a significant amount of recent work on the case when all workers have identical data distributions.
\citet{blanchard2017machine} initiated the study of Byzantine robust learning and proposed a distance-based aggregation approach \krum and extended to \citep{mhamdi2018hidden,Damaskinos:265684}.
\citet{yin2018byzantinerobust} propose to use and analyze the coordinate-wise median (\cm), and \citet{pillutla2019robust} use approximate geometric median. \citet{bernstein2018signsgd} propose to use the signs of gradients and then aggregate them by majority vote, however, \citet{karimireddy2019error} show that it may fail to converge. Most recently, \cite{alistarh2018byzantine,allen2020byzantine,mhamdi2021momentum,karimireddy2020learning} showcase how to use past gradients to more accurately filter \iid Byzantine workers and specifically \textit{time-coupled} attacks. In particular, our work builds on top of \citep{karimireddy2020learning} and non-trivially extends to the \noniid setting.

\textbf{IID vs. Non-IID attacks.} For the \iid setting, the state-of-the-art attacks are \textit{time-coupled} attacks~\citep{baruch2019little,xie2019fall}. These attacks introduce a small but consistent bias at every step which is hard to detect in any particular round, but accumulates over time and eventually leads to divergence, breaking most prior robust methods.
Our work focuses on developing attacks (and defenses) which specifically take advantages of the \noniid setting. The \noniid setting also enables targeted \textit{backdoor} attacks which are designed to take advantage of heavy-tailed data \citep{bagdasaryan2018backdoor,bhagoji2018analyzing}. However, this is a challenging and open problem~\citep{sun2019can,wang2020attack}. Our focus is on the overall accuracy of the trained model, not on any subproblem.
%

\textbf{Non-IID defenses.} The \noniid defenses are relatively under-examined.
\citet{ghosh2019robust,9054676} use an outlier-robust clustering method. When the server has the entire training dataset, the \noniid-ness is automatically addressed \citep{xie2018zeno,chen2018draco,rajput2019detox}. Typical examples are parallel training of neural networks on public cloud, or volunteer computing \citep{Meeds_2015,miura2015implementation}. Note that \citet{rajput2019detox} use hierarchical aggregation over "vote group" which is similar to the bucketing techniques but their results are limited to the \iid setting.
However, none of these methods are applicable to the standard federated learning.
This is partially tackled in \citep{data2020byzantine,data2020byzantineFL} who analyze spectral methods for robust optimization. However, these methods require $\Omega(d^2)$ time, making them infeasible for large scale optimization.
\citet{li2019rsa} proposes an SGD variant (\rsa) with additional $\ell_p$ penalty which only works for strongly convex objectives.
In an independent recent work, \citet{acharya2021robust} analyze geometric median (GM) on \noniid data using sparsified gradients. However, they do not defend against time coupled attacks, and their analysis neither proves convergence to the optimum nor recovers the standard rate of SGD when $\delta\!\rightarrow\!0$. In contrast, our analysis of GM addresses both issues and is more general.
For decentralized training with \noniid data, a parallel work \citep{el2021collaborative} considers asynchronous communication and unconstrained topologies and tolerates a maximum number of Byzantine workers in their setting. However, no convergence rate is given. \citet{he2022byzantine} consider decentralized training on constrained topologies and establish the consensus and convergence theory for a clipping based algorithm which tolerates a $\delta$-fraction of Byzantine workers, limited by the spectral gap of the topology.
Finally, \citet{pmlr-v139-yang21e} propose to use bucketing for asynchronous Byzantine learning which is very similar to the bucketing trick proposed in this paper for \noniid setup. In \Cref{ssec:additional} we further compare these two methods in terms of purposes, techniques, and analysis.
\footnote{The previous version of this work uses resampling which has identical performance as bucketing. The detailed comparison is listed in \Cref{ssec:resampling_vs_bucketing}.
}

\paragraph{Strong growth condition.} The assumption that $\E_{j \sim \cG}\norm{\nabla f_j(\xx) - \nabla f(\xx)}^2 \leq B^2 \norm{\nabla f(\xx)}^2$ for some $B \geq 0$ is also referred to as the strong growth condition \citep{schmidt2013fast}. This has been extensively used to analyze and derive optimization algorithms for deep learning \citep{schmidt2013fast,ma2018power,vaswani2019fast,vaswani2019painless,meng2020fast}. This line of work shows that the strong growth assumption is both realistic and (perhaps more importantly) useful in understanding optimization algorithms in deep learning. However, this is stronger than the \emph{weak} growth condition which states that $\E_{j \sim \cG}\norm{\nabla f_j(\xx) - \nabla f(\xx)}^2 \leq B^2(f(\xx) - f^\star)$ for some $B \geq 0$. For a smooth function $f$, the strong growth condition always implies the weak growth condition. Further, for smooth convex functions this is equivalent to assuming that all the workers functions $\{f_i\}$ share a common optimum, commonly known as interpolation. Our work uses the stronger version of the growth condition and it remains open to extend our results to the weaker version. This latter condition is strictly necessary for heterogeneous Byzantine optimization~\citep{gupta2020resilience}.


\section{Attacks against existing aggregation schemes}\label{sec:attacks}
In this section we show that when the data across the workers is heterogeneous (\noniid), then we can design simple new attacks which take advantage of the heterogeneity, leading to the failure of existing aggregation schemes. 
We study three representative and widely used defenses:

\textbf{Krum.} For $i\neq j$, let $i\rightarrow j$ denote that~$\xx_j$ belongs to the $n-q-2$ closest vectors to $\xx_i$. Then,
\vspace{-1mm}
\[
    \krum(\xx_1, \ldots, \xx_n) := \argmin_i \tsum_{i\rightarrow j}\|\xx_i-\xx_j\|^2 \,.\vspace{-3mm}
\]
Krum is computationally expensive, requiring $\cO(n^2)$ work by the server \citep{blanchard2017machine}.
\textbf{CM.} Coordinate-wise median computes for the $k$th coordinate:
\vspace{-1mm}
\[
    [\cm(\xx_1, \ldots, \xx_n)]_k := \text{median}([\xx_1]_k, \dots, [\xx_n]_k) = \argmin_{i} \tsum_{j=1}^n | [\xx_i]_k -[\xx_j]_k| \,.\vspace{-3mm}
\]
Coordinate-wise median is fast to implement requiring only $\cO(n)$ time \citep{Chen_2017}.

\textbf{RFA.} Robust federated averaging (\rfa) computes the geometric median
\vspace{-1mm}
\[
    \rfa(\xx_1, \ldots, \xx_n) := \argmin_{\vv} \tsum_{i=1}^n \|\vv-\xx_i\|_2 \,.\vspace{-2mm}
\]
While the geometric median has no closed form solution, \citep{pillutla2019robust} approximate it using multiple iterations of smoothed Weiszfeld algorithm, each of which requires $\cO(n)$ computation.

\vspace{-4mm}
\subsection{Failure on imbalanced data without Byzantine workers}\label{ssec:representative}
\vspace{-2mm}
We show that when the data amongst the workers is imbalanced, existing aggregation rules \emph{fail} even in the \emph{absence} of any Byzantine workers.
Algorithms like \krum select workers who are \emph{representative} of a majority of the workers by relying on statistics such as pairwise differences between the various worker updates.
Our key insight is that when the data across the workers is heterogeneous, there is no single worker who is representative of the whole dataset. This is because each worker computes their local gradient over vastly different local data.

\textbf{Example.} Suppose that there are $2n+1$ workers with worker $i$ holding $(-1)^i \in \{\pm 1\}$. This means that the true mean is $\approx 0$, but \krum, \cm, and \rfa will output $\pm 1$. This motivates our next attack.

Hence, for convergence it is important to not only select a good (non-Byzantine) worker, but also ensure that each of the good workers is selected with roughly equal frequency.
In \Cref{tab:noattack:naive}, we demonstrate failures of such aggregators by training on MNIST with $n\!=\!20$ and no attackers ($\delta\!=\!0$). We construct an imbalanced dataset where each successive class has only a fraction of samples of the previous class.
We defer details of the experiments to \Cref{sec:additional_exps}.
As we can see, \krum, \cm and \rfa match the ideal performance of SGD in the \iid case, but only attain less than 90\% accuracy in the \noniid case. This corresponds to learning only the top 2--3 classes and ignoring the rest.

A similar phenomenon was observed when using batch-size 1 in the \iid case by \citep{karimireddy2020learning}. However, in the \iid case this can be easily overcome by increasing the batch-size. In contrast, when the data across the works is \noniid (e.g. split by class), increasing the batch-size does \emph{not} make the worker gradients any more similar and there remains a big drop in performance. Finally, note that in \Cref{tab:noattack:naive} a hitherto new algorithm (\cclip) maintains its performance both in the \iid and the \noniid setting. We will explore this in more detail in \Cref{sec:bucketing}.


\begin{table}[t]
    \scriptsize
    \parbox{.45\linewidth}{
        \centering
        \captionsetup{font=small}
        \vspace{-4mm}
        \caption{Test accuracy (\%) with no Byzantine workers ($\delta\!=\!0$) on imbalanced data.}
        \vspace{-4mm}
        \label{tab:noattack:naive}
        \begin{tabular}{lcc}
\toprule
Aggr& \iid& {\noniid}\\\midrule
\textsc{Avg} & $98.79\!\pm\!0.10$ & $98.75\!\pm\!0.02$ \\
\krum & $97.95\!\pm\!0.25$ & $89.90\!\pm\!4.75$ \\
\cm & $97.72\!\pm\!0.22$ & $80.36\!\pm\!0.05$ \\
\rfa & $98.62\!\pm\!0.08$ & $82.60\!\pm\!0.84$ \\
\cclip & $98.78\!\pm\!0.10$ & $98.78\!\pm\!0.06$ \\
\bottomrule
\end{tabular}

    }
    \hfill
    \parbox{.45\linewidth}{
        \centering
        \captionsetup{font=small}
        \vspace{-4mm}
        \caption{Test accuracy (\%) under mimic attack with $\delta=0.2$ fraction of Byzantine workers.}
        \vspace{-4mm}
        \label{tab:mimic_attack:naive}
        \begin{tabular}{lcc}
\toprule
Aggr& \iid& {\noniid}\\\midrule
\textsc{Avg} & $93.20\!\pm\!0.21$ & $92.73\!\pm\!0.32$ \\
\krum & $90.36\!\pm\!0.25$ & $37.33\!\pm\!6.78$ \\
\cm & $90.80\!\pm\!0.12$ & $64.27\!\pm\!3.70$ \\
\rfa & $92.92\!\pm\!0.25$ & $78.93\!\pm\!9.27$ \\
\cclip & $93.16\!\pm\!0.22$ & $91.53\!\pm\!0.06$ \\
\bottomrule
\end{tabular}

    }
\end{table}

\subsection{Mimic attack on balanced data}\label{ssec:mimic}
Motivated by how data imbalance could lead to consistent errors in the aggregation rules and significant loss in accuracy, in this section, we will propose a new attack \emph{mimic} which specifically tries to maximize the perceived data imbalance even if the original data is balanced.

\textbf{Mimic attack.} All Byzantine workers pick a good worker (say $i_\star$) to mimic and copy its output (${\xx_{i_\star}^t}$). This inserts a consistent bias towards over-emphasizing worker $i_\star$ and thus under-representing other workers. Since the attacker simply mimics a good worker, it is impossible to distinguish it from a real worker and hence it cannot be filtered out. Indeed, the target $i_\star$ can be any fixed good worker. In \Cref{sec:app-mimic}, we present an empirical rule to choose $i_\star$ and include a simple example demonstrating how median based aggregators suffer from the heterogeneity under mimic attack.

\Cref{tab:mimic_attack:naive} shows the effectiveness of mimic attack even when the fraction of Byzantine nodes is small (i.e. $n=25$, $|\cB|=5$). Note that this attack specifically targets the \noniid nature of the data---all robust aggregators maintain their performance in the \iid setting and only suffer in the \noniid setting. Their performance is in fact worse than even simply averaging. As predicted by our example, \krum and \cm have the worst performance and \rfa performs slightly better. We will discuss the remarkable performance of \cclip in the next section.


\section{Constructing an agnostic robust aggregator using bucketing} %
\label{sec:bucketing}

In \Cref{sec:attacks} we demonstrated how existing aggregation rules fail in realistic \noniid scenarios, with and without attackers. In this section, we show how using bucketing can provably fix such aggregation rules.
The underlying reason for this failure, as we saw previously, is that the existing methods fixate on the contribution of only the most likely worker, and ignore the contributions from the rest.
To overcome this issue, we propose to use bucketing which `mixes' the data from all the workers thereby reducing the chance of any subset of the data being consistently ignored.

\subsection{Bucketing algorithm}

Given $n$ inputs $\xx_1, \dots, \xx_n$, we perform \textit{$s$-bucketing} which randomly partitions them into $\lceil n/s\rceil$ buckets with each bucket having no more than $s$ elements. Then, the contents of each bucket are averaged to construct $\{\yy_1, \dots, \yy_{\lceil n/s\rceil}\}$ which are then input to an aggregator $\aggr$. The details are summarized in \Cref{algo:rswor}.
The key property of our approach is that after bucketing, the resulting set of averaged  $\{\yy_1, \dots, \yy_{\lceil n/s\rceil}\}$ are much more homogeneous (lower variance) than the original inputs. Thus, when fed into existing aggregation schemes, the chance of success increases. We formalize this in the following simple lemma.

\begin{algorithm}[!t]
    \caption{Robust Aggregation (\ragg) using bucketing}
    \label{algo:rswor}
    \begin{algorithmic}[1]
        \State \textbf{input} $\{\xx_1, \dots, \xx_n\}$, $s\in\N$, aggregation rule \aggr
        \State pick random permutation $\pi$ of $[n]$
        \State compute $\yy_i \leftarrow \frac{1}{s} \sum_{k = (i-1)\cdot s + 1}^{\min(n\,,\, i\cdot s)} \xx_{\pi(k)}$  for $i =\{1,\ldots, {\lceil n/s \rceil}\}$
        \State \textbf{output} $\hat\xx \leftarrow \aggr(\yy_1, \dots, \yy_{\lceil n/s \rceil})$  \hfill // aggregate after bucketing
    \end{algorithmic}
    \vspace{-2mm}
\end{algorithm}

\begin{lemma}[Bucketing reduces variance]\label{lemma:bucketing}
    Suppose we are given $n$ independent (but not identical) random vectors $\{\xx_1, \dots, \xx_n\}$ such that a good subset $\cG \subseteq [n]$ of size at least $\abs{\cG} \geq n(1 - \delta)$ satisfies:\vspace{-3mm}
    \[
        \E\norm{\xx_i - \xx_j}^2 \leq \rho^2\,, \quad \text{ for any fixed } i, j \in \cG\,.
    \]
    \vspace{-1mm}Define $\bar\xx := \frac{1}{\abs{\cG}} \sum_{j \in \cG} \xx_j$.
    Let the outputs after $s$-bucketing be $\{\yy_1, \dots, \yy_{\lceil n/s\rceil}\}$ and denote $\tilde\cG \subseteq \{1,\ldots,\lceil n/s\rceil\}$ as a good bucket set where a good bucket contains only elements belonging to $\cG$. Then $\abs{\tilde\cG} \geq \lceil n/s\rceil (1 - \delta s)$ satisfies \vspace{-2mm}
    \[
        \E[\yy_i] = \E[\bar\xx] \quad \text{ and } \quad \E\norm{\yy_i - \yy_j} \leq {\rho^2}/{s} \quad \text{for any fixed } i,j \in \tilde\cG\,.
    \]
\end{lemma}
The expectation in the above lemma is taken both over the random vectors as well as over the randomness of the bucketing procedure.
\begin{remark}
    \Cref{lemma:bucketing} proves that after our bucketing procedure, we are left with outputs $\yy_i$ which have i) pairwise variance reduced by $s$, and ii) potentially $s$ times more fraction of Byzantine vectors.
    Hence, bucketing trades off increasing influence of Byzantine inputs against having more homogeneous vectors. Using $s=1$ simply shuffles the inputs and leaves them otherwise unchanged.
\end{remark}

\subsection{Agnostic robust aggregation}
We now define what it means for an agnostic robust aggregator to succeed.
\begin{definition}[$(\delta_{\max},c)$-\ragg]\label{definition:robust-agg}
    Suppose we are given input $\{\xx_1, \dots, \xx_n\}$ of which a subset $\cG$ of size at least $\abs{\cG} > (1-\delta)n$  for $\delta \leq \delta_{\max} < 0.5$ and satisfies
    \(
    \E\norm{\xx_i - \xx_j}^2 \leq \rho^2 \,. \)
    Then, the output $\hat\xx$ of a Byzantine robust aggregator satisfies:\vspace{-1mm}
    \[
        \E\norm{\hat\xx - \bar\xx}^2 \leq c \delta \rho^2  \quad \text{ where }\quad \hat\xx = \ragg_\delta(\xx_1, \dots, \xx_n)\,.\vspace{-1mm}
    \]
    Further, $\ragg$ does not need to know $\rho^2$ (only $\delta$), and automatically adapts to any value $\rho^2$. \vspace{-0mm}
\end{definition}
Our robust aggregator is parameterized by $\delta_{\max} < 0.5$ which denotes the maximum amount of Byzantine inputs it can handle, and a constant $c$ which determines its performance. If $\delta = 0$, i.e. when there are no Byzantine inputs, we are guaranteed to \textit{exactly} recover the true average $\bar\xx$. Exact recovery is also guaranteed when $\rho = 0$ since in that case it is easy to identify the good inputs since they are all equal and in majority.  When both $\rho > 0$ and $\delta > 0$, we recover the average up to an additive error term.
We also require that the robust aggregator is \emph{agnostic} to the value of $\rho^2$ and automatically adjusts its output to the current $\rho$ during training. The aggregator can take $\delta$ as an input though. This property is very useful in the context of Byzantine robust optimization since the variance $\rho^2$ keeps changing over the training period, whereas the fraction of Byzantine workers $\delta$ remains constant. This is a major difference from the definition used in \citep{karimireddy2020learning}.
Note that \Cref{definition:robust-agg} is defined for both homogeneous and heterogeneous data.

We next show that aggregators which we saw were not robust in \Cref{sec:attacks}, can be made to satisfy \Cref{definition:robust-agg} by combining with bucketing.
\begin{theorem}\label{thm:bucketing}
    Suppose we are given $n$ inputs $\{\xx_1, \dots, \xx_n\}$ satisfying properties in \Cref{lemma:bucketing} for some $\delta \le \delta_{\max}$, with $\delta_{\max}$ to be defined. Then, running \Cref{algo:rswor} with $s = \lfloor \nicefrac{\delta_{\max}}{\delta} \rfloor$ yields the following:
    \begin{itemize}[nosep]
        \item Krum: $\E \norm{\krum \circ \resample (\xx_1, \dots, \xx_n) - \bar\xx}^2 \leq \cO( \delta \rho^2)$ with $\delta_{\max} < \nicefrac{1}{4}$.
        \item Geometric median: $\E \norm{\rfa \circ \resample (\xx_1, \dots, \xx_n) - \bar\xx}^2 \leq \cO( \delta \rho^2)$ with $\delta_{\max} < \nicefrac{1}{2}$.
        \item Coordinate-wise median: $\E \norm{\cm \circ \resample (\xx_1, \dots, \xx_n) - \bar\xx}^2 \leq \cO(d \delta \rho^2)$ with $\delta_{\max} < \nicefrac{1}{2}$\,.
    \end{itemize}
\end{theorem}
Note that all these methods satisfy our notion of an \emph{agnostic} Byzantine robust aggregator (\Cref{definition:robust-agg}). This is because both our bucketing procedures as well as the underlying aggregators are independent of $\rho^2$. Further, our error is $\cO(\delta\rho^2)$ and is information theoretically optimal, unlike previous analyses (e.g. \citet{acharya2021robust}) who had an error of $\cO(\rho^2)$.

The error of \cm depends on the dimension $d$ which is problematic when $d \ggg n$. However, we suspect this is because we measure stochasticity using Euclidean norms instead of coordinate-wise. In practice, we found that \cm often outperforms \krum, with \rfa outperforming them both.  Note that we select $s = \lfloor\nicefrac{\delta_{\max}}{\delta}\rfloor$ to ensure that after bucketing, we have the maximum amount of Byzantine inputs tolerated by the method with $(s\delta) = \delta_{\max}$.

\begin{remark}[1-step Centered clipping]\label{rem:cclip}
    The 1-step centered clipping aggregator (\cclip) given a clipping radius $\tau$ and an initial guess $\vv$ of the average $\bar\xx$ performs\vspace{-2mm}
    \[
        \cclip(\xx_1, \dots, \xx_n) = \vv + \tfrac{1}{n}\tsum_{i \in [n]} (\xx_n - \vv) \min\rbr{1 \,,\, {\tau}/{\norm{\xx_n - \vv}_2}}\,.\vspace{-3mm}
    \]
    \citet{karimireddy2020learning} prove that \cclip even without bucketing satisfies \Cref{definition:robust-agg} with $\delta_{\max} = 0.1$, and $c=\cO(1)$. This explains its good performance on \noniid data in \Cref{sec:attacks}. However, \cclip is \emph{not agnostic} since it requires clipping radius $\tau$ as an input which in turn depends on $\rho^2$. Devising a version of \cclip which automatically adapts its clipping radius is an important open question. Empirically however, we observe that simple rules for setting $\tau$ work quite well---we always use $\tau = \frac{10}{1 - \beta}$ in our limited experiments where $\beta$ is the coefficient of momentum.
\end{remark}


While we have shown how to construct a robust aggregator which satisfies some notion of a robustness, we haven't yet seen how this affects the Byzantine robust \emph{optimization} problem. We investigate this question theoretically in the next section and empirically in \Cref{sec:experiments}.

\section{Robust \noniid optimization using a robust aggregator}\label{sec:analysis}

In this section, we study the problem of optimization in the presence of Byzantine workers and heterogeneity, given access to any robust aggregator satisfying \Cref{definition:robust-agg}. We then show that data heterogeneity makes Byzantine robust optimization especially challenging and prove lower bounds for the same. Finally, we see how mild heterogeneity, or sufficient overparameterization can circumvent these lower bounds, obtaining convergence to the optimum.

\vspace{-2mm}
\begin{algorithm}[!thb]
    \caption{Robust Optimization using any Agnostic Robust Aggregator}\label{algo:bucketing_sgd}
    \renewcommand{\algorithmiccomment}[1]{#1}
    \algblockdefx[NAME]{ParallelFor}{ParallelForEnd}[1]{\textbf{for} #1 \textbf{in parallel}}{}
    \algnotext{ParallelForEnd}
    \algblockdefx[NAME]{NoEndFor}{NoEndForEnd}[1]{\textbf{for} #1 \textbf{do}}{}
    \algnotext{NoEndForEnd}
    \begin{algorithmic}[1] 
        \Require \ragg, $\eta$, $\beta$
        \NoEndFor{$t=1,...$}
        \ParallelFor{worker $i\in [n]$}
        \State $\gg_i \leftarrow \nabla F_i(\xx, \xiv_i)$ and $\mm_i \leftarrow (1-\beta)\gg_i + \beta \mm_i $ \hfill // worker momentum
        \State \textbf{send} $\mm_i$ if $i \in \cG$, else send $*$ if Byzantine
        \ParallelForEnd
        \State $\hat\mm$ = \ragg($\mm_1, \dots, \mm_n$) and $\xx \leftarrow \xx - \eta\hat\mm$. \hfill // update params using robust aggregate
        \NoEndForEnd
    \end{algorithmic}
\end{algorithm}

\vspace{-4mm}
\subsection{Algorithm description}
\vspace{-2mm}
In \Cref{sec:bucketing} we saw that bucketing could tackle heterogeneity across the workers by reducing $\zeta^2$. However, there still remains variance $\sigma^2$ in the gradients within each worker since each worker uses stochastic gradients. To reduce the effect of this variance, we rely on worker momentum. Each worker sends their local worker momentum vector $\mm_i$ to be aggregated by $\ragg$ instead of $\gg_i$:\vspace{-2mm}
\begin{align*}
    \mm_i^t & = \beta \mm_i^{t-1} + (1 - \beta)\gg_i(\xx^{t-1}) \quad \text{ for every } i \in \cG\,, \\
    \xx^t   & = \xx^{t-1} - \eta \ragg(\mm_1^t, \dots, \mm_n^t)\,.
    \vspace{-2mm}
\end{align*}
This is equivalent to the usual momentum description up to a rescaling of step-size $\eta$. Intuitively, using worker momentum $\mm_i$ averages over $\nicefrac{1}{(1-\beta)}$ independent stochastic gradients $\gg_i$ and thus reduces the effect of the within-worker-variance $\sigma^2$ \citep{karimireddy2020learning}. Note that the resulting $\{\mm_i\}$ are \emph{still heterogeneous} across the workers. This heterogeneity is the key challenge we face.

\vspace{-4mm}
\subsection{Convergence rates}
\vspace{-2mm}
We now turn towards proving convergence rates for our bucketing aggregation method \Cref{algo:rswor} based on any existing aggregator $\aggr$. We will assume that for any fixed $i \in \cG$
\begin{equation}\label{eqn:asm-het}
    \E_{\xiv_i}\norm{\gg_i(\xx) - \nabla f_i(\xx)}^2 \leq \sigma^2 \text{ and }  \E_{j \sim \cG}\norm{\nabla f_j(\xx) - \nabla f(\xx)}^2 \leq \zeta^2 \,, \quad \forall \xx\,.
\end{equation}
This first condition bounds the variance of the stochastic gradient within a worker whereas the latter is a standard measure of inter-client heterogeneity in federated learning \citep{yu2019parallel,khaled2020tighter,karimireddy2020scaffold}. Under these conditions, we can prove the following.
\begin{theorem}\label{thm:convergence-general}
    Suppose we are given a $(\delta_{\max}, c)$-\ragg satisfying \Cref{definition:robust-agg}, and $n$ workers of which a subset $\cG$ of size at least $\abs{\cG} \geq n(1 - \delta)$ faithfully follow the algorithm for $\delta \leq \delta_{\max}$. Further, for any good worker $i \in \cG$ let $f_i$ be a possibly non-convex function with $L$-Lipschitz gradients, and the stochastic gradients on each worker be independent, unbiased and satisfy \eqref{eqn:asm-het}.
    Then, for $F^0 := f(\xx^0) - f^\star$, the output of \Cref{algo:bucketing_sgd} satisfies \vspace{-2mm}
    \begin{align*}
        \tfrac{1}{T}\tsum_{t=1}^T \E \norm{\nabla f(\xx^{t-1})}^2 & \leq
        \cO \rbr[\Big]{ c \delta \zeta^2 +
            \sigma\sqrt{\tfrac{LF^0}{T} (c\delta + \nicefrac{1}{n})} + \tfrac{L F^0}{T} } \,.
    \end{align*}
\end{theorem}
\begin{remark}[Unified proofs] \Cref{rem:cclip} shows that \cclip is a robust aggregator, and \Cref{thm:bucketing} shows \krum, \rfa, and \cm  on combining with sufficient bucketing are all robust aggregators satisfying \Cref{definition:robust-agg}. Most of these methods had no end-to-end convergence guarantees prior to our results. Thus, \Cref{thm:convergence-general} gives the first unified analysis in both the \iid and \noniid settings.
\end{remark}
When $\delta \rightarrow 0$ i.e. as we reduce the number of Byzantine workers, the above rate recovers the optimal $\cO(\frac{\sigma}{\sqrt{Tn}})$ rate for non-convex SGD and even has linear speed-up with respect to the $n$ workers. In contrast, all previous algorithms for \noniid data (e.g. \citep{data2020byzantineFL,acharya2021robust}) do not improve their rates for decreasing values of $\delta$. This is also empirically reflected in \Cref{ssec:representative}, where these algorithms are shown to fail even in the absence of Byzantine workers ($\delta = 0$).

Further, when $\zeta = 0$ the rate above simplifies to $\cO(\frac{\sigma }{\sqrt{T}} \cdot \sqrt{c\delta + \nicefrac{1}{n}})$ which matches the \iid Byzantine robust rates of \citep{karimireddy2020learning}. In both cases we converge to the optimum and can make the gradient arbitrarily small. However, when $\delta > 0$ and $\zeta > 0$, \Cref{thm:convergence-general} only shows convergence to a radius of $\cO(\sqrt{\delta}\zeta)$ and not to the actual optimum. We will next explore this limitation.

\vspace{-2mm}
\subsection{Lower bounds and the challenge of heterogeneity}
\vspace{-2mm}
Suppose worker $j$ sends us an update which looks `weird' and is very different from the updates from the rest of the workers. This may be because worker $j$ might be malicious and their update represents an attempted attack. It may also be because worker $j$ has highly \emph{non-representative data}. In the former case the update should be ignored, whereas in the latter the update represents a valuable source of specialized data. However, it is impossible for the server to distinguish between the two situations.
The above argument can in fact be formalized to prove the following lower bound.

\begin{theorem}\label{thm:lower-bound}
    Given any optimization algorithm \alg, we can find $n$ functions $\{f_1(x), \dots , f_n(x)\}$ of which at least $(1-\delta)n$ are good (belong to $\cG$), 1-smooth, $\mu$-strongly convex functions, and satisfy $\E_{i \sim \cG}\norm{\nabla f_i(x) - \nabla f(x)}^2 \leq \zeta^2$ such that the output of \alg has an error at least \vspace{-3mm}
    \[
        \E [f(\alg(f_1, \dots, f_n)) - f^\star] \geq \Omega\rbr*{\tfrac{\delta \zeta^2}{\mu}} \quad \text{ and }\quad \E \norm{ \nabla f(\alg(f_1, \dots, f_n)) }^2 \geq \Omega\rbr*{\delta \zeta^2}\,.
    \]
\end{theorem}
\vspace{-2mm}
The expectation above is over the potential randomness of the algorithm. This theorem
unfrotunately implies that it is impossible to converge to the true optimum in the presence of Byzantine workers. Note that the above lower bound is information theoretic in nature and is independent of how many gradients are computed or how long the algorithm is run.
\begin{remark}[Matches lower bound]
    Suppose that we satisfy the heterogeneity condition \eqref{eqn:asm-het} with $\zeta^2 > 0$ and $\sigma =0$. Then, the rate in \Cref{thm:convergence-general} can be simplified to $\cO\rbr[\big]{\delta \zeta^2 + 1/T}$. While the second term in this decays to 0 with $T$, the first term remains, implying that we only converge to a radius of $\sqrt{\delta}\zeta$ around the optimum. However, this matches our lower bound result from \Cref{thm:lower-bound} and hence is in general unimprovable.
\end{remark}

This is a very strong negative result and seems to indicate that Byzantine robustness might be impossible to achieve in real world federated learning. This would be major stumbling block for deployment since the system would provably be vulnerable to attackers. We will next carefully examine the lower bound and will attempt to circumvent it.

\vspace{-2mm}
\subsection{Circumventing lower bounds using overparameterization}\label{ssec:overparameterization}
\vspace{-2mm}
We previously saw some strong impossibility results posed by heterogeneity. In this section, we show that while indeed in the worst case being robust under heterogeneity is impossible, we may still converge to the true optimum under more realistic settings. We consider an alternative bound of \eqref{eqn:asm-het}:\vspace{-2mm}
\begin{equation}\label{eqn:asm-overpara}
    \E_{j \sim \cG}\norm{\nabla f_j(\xx) - \nabla f(\xx)}^2 \leq B^2\norm{\nabla f(\xx)}^2 \,, \quad \forall \xx\,.
\end{equation}
Note that at the optimum $\xx^\star$ we have $\nabla f(\xx^\star) = 0$, and hence this assumption
implies that $\nabla f_j(\xx^\star) = 0$ for all $j \in \cG$. This is satisfied if the model is \emph{sufficiently over-parameterized} and typically holds in most realistic settings \citep{vaswani2019fast}.

\begin{theorem}\label{thm:overparam-convergence}
    Suppose we are given a $(\delta_{\max}, c)$-\ragg and $n$ workers with loss functions $\{f_1, \dots, f_n\}$ satisfying the conditions in \Cref{thm:convergence-general} with $\delta \leq \delta_{\max}$ and \eqref{eqn:asm-overpara} for some $B^2 < \frac{1}{60c\delta}$. Then, for $F^0 := f(\xx^0) - f^\star$, the output of \Cref{algo:bucketing_sgd} satisfies \vspace{-2mm}
    \begin{align*}
        \tfrac{1}{T}\tsum_{t=1}^T \E \norm{\nabla f(\xx^{t-1})}^2 & \leq
        \cO \rbr[\Big]{\tfrac{1}{1 \!-\! 60c\delta B^2} \cdot \rbr[\Big]{
                \sigma\sqrt{\tfrac{LF^0}{T} (c\delta + \nicefrac{1}{n})} + \tfrac{L F^0}{T}}}\,.
    \end{align*}
\end{theorem}
\begin{remark}[Overparameterization fixes convergence] The rate in \Cref{thm:overparam-convergence} not only goes to 0 with~$T$, but also matches that of the optimal \iid rate of $\cO(\frac{\sigma }{\sqrt{T}} \cdot \sqrt{c\delta + \nicefrac{1}{n}})$ \citep{karimireddy2020learning}. Thus, using a stronger heterogeneity assumption allows us to circumvent lower bounds for the \noniid case and converge to a good solution even in the presence of Byzantine workers. This is the first result of its kind, and takes a major step towards realistic and practical robust algorithms.
\end{remark}

In the overparameterized setting, we can be sure that we will able to \emph{simultaneously} optimize all worker's losses. Hence, over time the agreement between all worker's gradients increases. This in turn makes any attempts by the attackers to derail training stand out easily, especially towards the end of the training. To take advantage of this increasing closeness, we need an aggregator which automatically adapts the quality of its output as the good workers get closer. Thus, the \emph{agnostic} robust aggregator is crucial to our overparameterized convergence result.
We empirically demonstrate the effects of overparameterization in \Cref{sssec:overparameterization}.

\begin{table}[t]
    \scriptsize
    \parbox{.45\linewidth}{
        \scriptsize
        \centering
        \vspace{-4mm}
        \caption{\Cref{tab:noattack:naive} + Bucketing (s=2).}
        \vspace{-4mm}
        \label{tab:noattack:resample}

        \begin{tabular}{lcc}
\toprule
Aggr& \iid& {\noniid}\\\midrule
\textsc{Avg} & $98.80\!\pm\!0.10$ & $98.74\!\pm\!0.02$ \\
\krum & $98.35\!\pm\!0.20$ & $93.27\!\pm\!0.10$ \\
\cm & $98.26\!\pm\!0.22$ & $95.59\!\pm\!0.89$ \\
\rfa & $98.75\!\pm\!0.14$ & $97.34\!\pm\!0.58$ \\
\cclip & $98.79\!\pm\!0.10$ & $98.75\!\pm\!0.02$ \\
\bottomrule
\end{tabular}

    }
    \hfill
    \parbox{.45\linewidth}{
        \centering
        \vspace{-4mm}
        \caption{\Cref{tab:mimic_attack:naive} + Bucketing (s=2).}
        \vspace{-4mm}
        \label{tab:mimic_attack:resample}
        \begin{tabular}{lcc}
\toprule
Aggr& \iid& {\noniid}\\\midrule
\textsc{Avg} & $93.17\!\pm\!0.23$ & $92.67\!\pm\!0.27$ \\
\krum & $91.64\!\pm\!0.30$ & $53.15\!\pm\!3.96$ \\
\cm & $91.91\!\pm\!0.24$ & $78.60\!\pm\!3.15$ \\
\rfa & $93.00\!\pm\!0.23$ & $91.17\!\pm\!0.51$ \\
\cclip & $93.17\!\pm\!0.23$ & $92.56\!\pm\!0.21$ \\
\bottomrule
\end{tabular}

    }
    \vspace{-4mm}
\end{table}

\vspace{-2mm}
\section{Experiments} %
\label{sec:experiments}

\vspace{-2mm}
In this section, we demonstrate the effects of bucketing on datasets distributed in a \noniid fashion. Throughout the section, we illustrate the tasks, attacks, and defenses by an example of training an MLP on a heterogeneous version of the MNIST dataset \citep{lecun1998gradient}. 
The dataset is sorted by labels and sequentially divided into equal parts among good workers; Byzantine workers have access to the entire dataset. 
Implementations are based on PyTorch \citep{paszke2019pytorch} and will be made publicly available.\footnote{
    The code is available at
    \href{https://github.com/epfml/byzantine-robust-noniid-optimizer}{this url}.}
We defer details of setup, implementation, and runtime to \Cref{sec:additional_exps}.

\vspace{-2mm}
\textbf{Bucketing against the attacks on \noniid data.}
In \Cref{sec:attacks} we have presented how heterogeneous data can lead to failure of existing robust aggregation rules. Here we apply our proposed bucketing with $s\!=\!2$ to the same aggregation rules, showing that bucketing overcomes the described failures. Results are presented in \Cref{tab:noattack:resample}.
Comparing \Cref{tab:noattack:resample} with \Cref{tab:noattack:naive}, bucketing improves the aggregators' top-1 test accuracy on long-tail and \noniid dataset by 4\% to 14\% and allows them to learn classes at the tail distribution. For \noniid balanced dataset, bucketing also greatly improves the performance of \krum and \cm and makes \rfa and \cclip close to ideal performance.
Similarly, combining aggregators with bucketing also performs much better on \noniid dataset under mimic attack. In \Cref{tab:mimic_attack:resample}, \rfa and \cclip recover \iid accuracy, and \krum, and \cm are improved by around 15\%.

\vspace{-2mm}
\textbf{Bucketing against general Byzantine attacks.}
In \Cref{fig:main}, we present thorough experiments on \noniid data over 25 workers with 5 Byzantine workers, under different attacks. In each subfigure, we compare an aggregation rule with its variant with bucketing.
The aggregation rules compared are \krum, \cm, \rfa, \cclip.
\begin{figure*}[!t]
    \vspace{-3mm}
    \centering
    \includegraphics[
        width=\linewidth
    ]{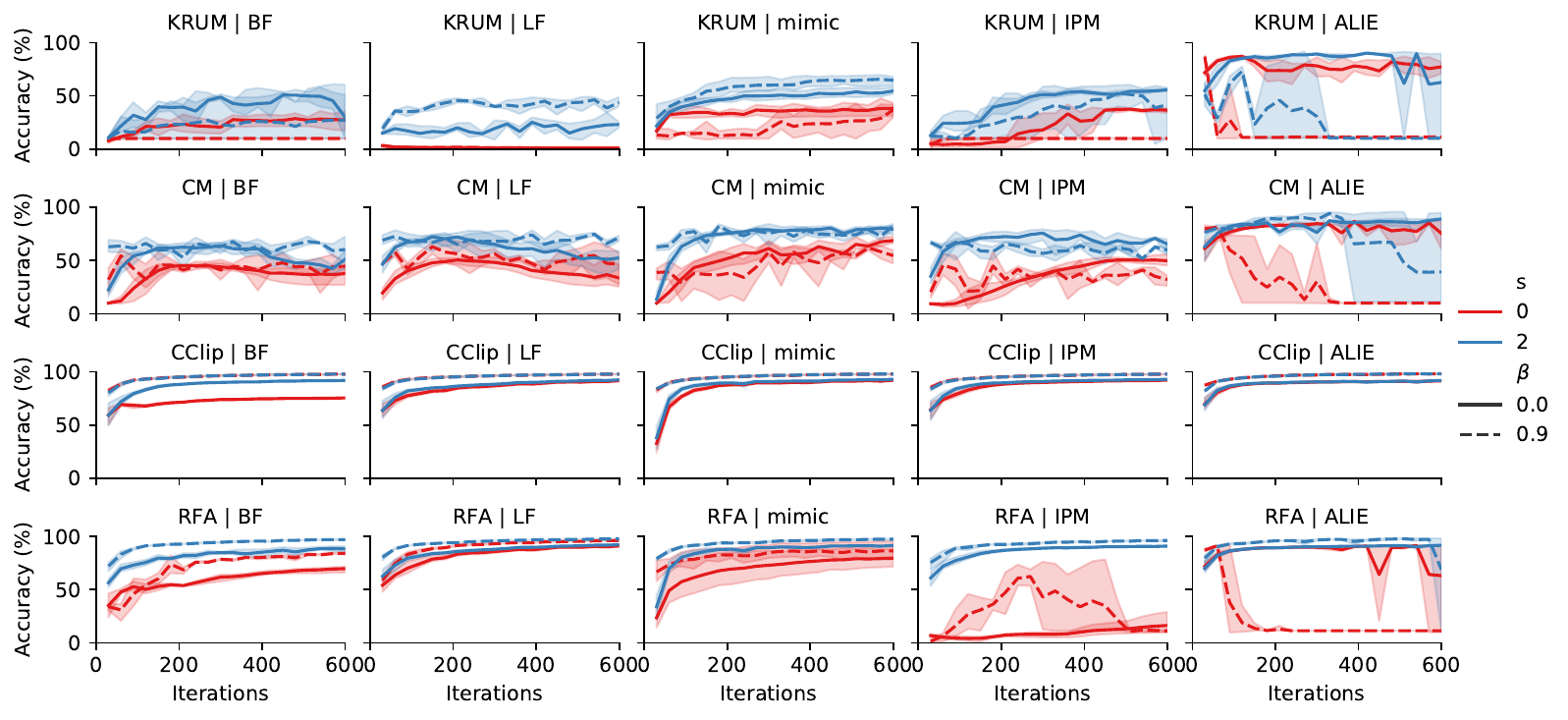}
    \vspace{-8mm}
    \caption{
        Top-1 test accuracies of \krum, \cm, \cclip, \rfa, under 5 attacks on \noniid datasets.
    }
    \vspace{-2mm}
    \label{fig:main}
\end{figure*}
%
5 different kinds of attacks are applied (one per column in the figure): bit flipping (BF), label flipping (LF), \textit{mimic} attack, as well as inner product manipulation (IPM) attack \citep{xie2019fall} and the ``a little is enough'' (ALIE) attack \citep{baruch2019little}.
\begin{itemize}[nosep]
    \item \textbf{Bit flipping}: A Byzantine worker
          sends $-\nabla f(\xx)$ instead of $\nabla f(\xx)$ due to hardware failures etc.

    \item \textbf{Label flipping}: Corrupt MNIST dataset by transforming labels by $\mathcal{T}(y) := 9 - y$.

    \item \textbf{Mimic}: Explained in \Cref{ssec:mimic}.

    \item \textbf{IPM}: The attackers send $-\frac{\epsilon}{|\cG|}\sum_{i\in\cG} \nabla f(\xx_i)$ where $\epsilon$ controls the strength of the attack.
    \item \textbf{ALIE}: The attackers estimate the mean $\mu_{\cG}$ and standard deviation $\sigma_{\cG}$ of the good gradients, and send $\mu_{\cG}-z \sigma_{\cG}$ to the server where $z$ is a small constant controlling the strength of the attack.
\end{itemize}
Both IPM and ALIE are the state-of-the-art attacks in the \iid distributed learning setups which takes advantage of the variances among workers. These attacks are much stronger in the \noniid setup. In the last two columns of \Cref{fig:main} we show that worker momentum and bucketing reduce such variance while momentum alone is not enough. Overall, \Cref{fig:main} shows that bucketing improves the performances of almost all aggregators under all kinds of attacks.
%
Note that $\tau$ of \cclip is not finetuned for each attack but rather fixed to $\frac{10}{1-\beta}$ for all attacks. This scaling is required because \cclip is \emph{not agnostic}. We defer the discussion to \Cref{sssec:clipping_radius}.


\begin{figure*}[!t]
    \centering
    \vspace{-1em}
    \begin{subfigure}[!t]{.4\textwidth}
        \centering
        \includegraphics[width=\linewidth]{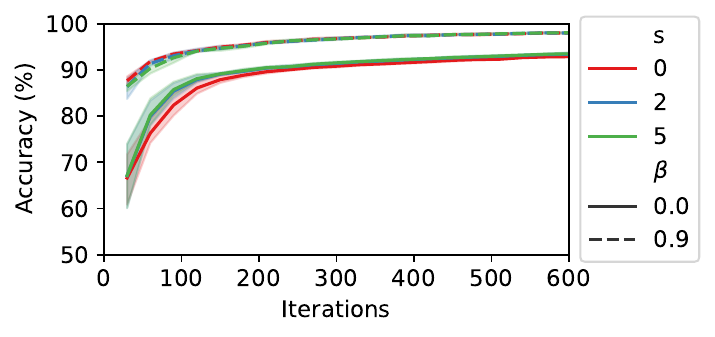}
        \vspace{-6mm}
        \caption{Fixed $q\!=\!5$ IPM attackers, varying $s$}
        \vspace{-2mm}
        \label{fig:hp:fix_f_vary_s}
    \end{subfigure}
    ~
    \begin{subfigure}[!t]{.4\textwidth}
        \centering
        \includegraphics[width=\linewidth]{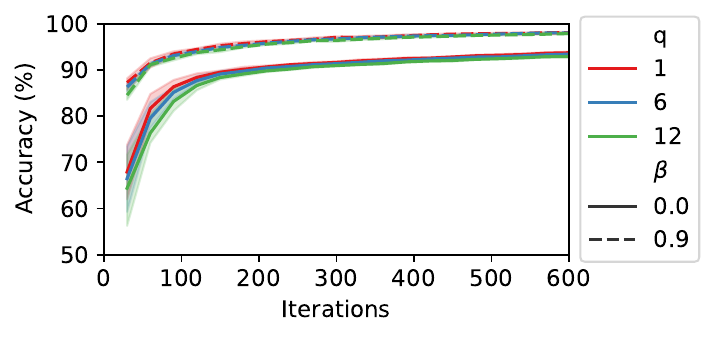}
        \vspace{-6mm}
        \caption{Fixed $s\!=\!2$, varying number $q$}
        \vspace{-2mm}
        \label{fig:hp:fix_s_vary_f}
    \end{subfigure}%
    \vspace{-2mm}
    \caption{Top-1 accuracies of \cclip with varying $q$ and $s$ when training on a cluster of $n\!=\!53$ nodes}
    \vspace{-8mm}
    \label{fig:hp}
\end{figure*}

\vspace{-2mm}
\textbf{Bucketing hyperparameter.}\label{ssec:bucketing_coefficient}
Finally we study the influence of $s$ and $q$ on the heterogeneous MNIST dataset. We use \cclip as the base aggregator and apply IPM attack. The \Cref{fig:hp:fix_f_vary_s} confirms that larger $s$ gives faster convergence but $s\!=\!2$ is sufficient. \Cref{fig:hp:fix_s_vary_f} shows that $s\!=\!2$ still behaves well when increasing $q$ close to 25\%. The complete evaluation of the results are deferred to \Cref{sec:additional_exps}.

\vspace{-2mm}
\textbf{Discussion.}
In all our experiments, we consistently observe: i) mild bucketing ($s\!=\!2$) improves performance, ii) worker momentum further stabilizes training, and finally iii) \cclip recovers the ideal performance. Given its ease of implementation, this leads us to strongly recommend using \cclip in practical federated learning to safeguard against actively malicious agents or passive failures.
\rfa combined with bucketing and worker momentum also nearly recovers ideal performance and can instead be used when a proper radius $\tau$ is hard to find.
Designing an \emph{automatic} and \emph{adaptively} clipping radius as well as its large scale empirical study is left for future work.

\vspace{-3mm}
\section{Conclusion}
\vspace{-3mm}

%

Heterogeneity poses unique challenges for Byzantine robust optimization. The first challenge is that existing defenses attempt to pick a ``representative'' update, which may not exist in the non-iid setting. This, we showed, can be overcome by using bucketing. A second more fundamental challenge is that it is difficult to distinguish between a ``weird'' but good worker from an actually Byzantine attacker. In fact, we proved strong impossibility results in such a setting. For this we showed how overparameterization (which is prevalent in real world deep learning) provides a solution, ensuring convergence to the optimum even in the presence of attackers. Together, our results yield a practical provably Byzantine robust algorithms for the \noniid setting.



\paragraph{Acknowledgement.} This project was supported by SNSF grant 200020\_200342.

\bibliography{iclr_robust-noniid}

\begin{thebibliography}{51}
\providecommand{\natexlab}[1]{#1}
\providecommand{\url}[1]{\texttt{#1}}
\expandafter\ifx\csname urlstyle\endcsname\relax
  \providecommand{\doi}[1]{doi: #1}\else
  \providecommand{\doi}{doi: \begingroup \urlstyle{rm}\Url}\fi

\bibitem[Acharya et~al.(2021)Acharya, Hashemi, Jain, Sanghavi, Dhillon, and
  Topcu]{acharya2021robust}
Anish Acharya, Abolfazl Hashemi, Prateek Jain, Sujay Sanghavi, Inderjit~S
  Dhillon, and Ufuk Topcu.
\newblock Robust training in high dimensions via block coordinate geometric
  median descent.
\newblock \emph{arXiv preprint arXiv:2106.08882}, 2021.

\bibitem[Alistarh et~al.(2018)Alistarh, Allen-Zhu, and
  Li]{alistarh2018byzantine}
Dan Alistarh, Zeyuan Allen-Zhu, and Jerry Li.
\newblock Byzantine stochastic gradient descent.
\newblock In \emph{NeurIPS - Advances in Neural Information Processing
  Systems}, pp.\  4613--4623, 2018.

\bibitem[Allen-Zhu et~al.(2021)Allen-Zhu, Ebrahimian, Li, and
  Alistarh]{allen2020byzantine}
Zeyuan Allen-Zhu, Faeze Ebrahimian, Jerry Li, and Dan Alistarh.
\newblock Byzantine-resilient non-convex stochastic gradient descent.
\newblock In \emph{ICLR 2021 - International Conference of Learning
  Representations}, 2021.

\bibitem[Bagdasaryan et~al.(2020)Bagdasaryan, Veit, Hua, Estrin, and
  Shmatikov]{bagdasaryan2018backdoor}
Eugene Bagdasaryan, Andreas Veit, Yiqing Hua, Deborah Estrin, and Vitaly
  Shmatikov.
\newblock How to backdoor federated learning.
\newblock In \emph{AISTATS - Proceedings of the Twenty Third International
  Conference on Artificial Intelligence and Statistics}, volume 108, pp.\
  2938--2948, 2020.

\bibitem[Baruch et~al.(2019)Baruch, Baruch, and Goldberg]{baruch2019little}
Moran Baruch, Gilad Baruch, and Yoav Goldberg.
\newblock {A Little Is Enough: Circumventing Defenses For Distributed
  Learning}.
\newblock \emph{NeurIPS}, 2019.

\bibitem[Bernstein et~al.(2018)Bernstein, Zhao, Azizzadenesheli, and
  Anandkumar]{bernstein2018signsgd}
Jeremy Bernstein, Jiawei Zhao, Kamyar Azizzadenesheli, and Anima Anandkumar.
\newblock {signSGD with majority vote is communication efficient and fault
  tolerant}.
\newblock \emph{arXiv preprint arXiv:1810.05291}, 2018.

\bibitem[Bhagoji et~al.(2018)Bhagoji, Chakraborty, Mittal, and
  Calo]{bhagoji2018analyzing}
Arjun~Nitin Bhagoji, Supriyo Chakraborty, Prateek Mittal, and Seraphin Calo.
\newblock Analyzing federated learning through an adversarial lens, 2018.

\bibitem[Blanchard et~al.(2017)Blanchard, El~Mhamdi, Guerraoui, and
  Stainer]{blanchard2017machine}
Peva Blanchard, El~Mahdi El~Mhamdi, Rachid Guerraoui, and Julien Stainer.
\newblock {Machine Learning with Adversaries: Byzantine Tolerant Gradient
  Descent}.
\newblock In \emph{NeurIPS - Advances in Neural Information Processing Systems
  30}, pp.\  119--129, 2017.

\bibitem[Bonawitz et~al.(2019)Bonawitz, Eichner, Grieskamp, Huba, Ingerman,
  Ivanov, Kiddon, Konecny, Mazzocchi, McMahan, et~al.]{bonawitz2019towards}
Keith Bonawitz, Hubert Eichner, Wolfgang Grieskamp, Dzmitry Huba, Alex
  Ingerman, Vladimir Ivanov, Chloe Kiddon, Jakub Konecny, Stefano Mazzocchi,
  H~Brendan McMahan, et~al.
\newblock Towards federated learning at scale: System design.
\newblock In \emph{SysML - Proceedings of the 2nd SysML Conference, Palo Alto,
  CA, USA}, 2019.

\bibitem[Chen et~al.(2018)Chen, Wang, Charles, and
  Papailiopoulos]{chen2018draco}
Lingjiao Chen, Hongyi Wang, Zachary Charles, and Dimitris Papailiopoulos.
\newblock Draco: Byzantine-resilient distributed training via redundant
  gradients.
\newblock \emph{arXiv preprint arXiv:1803.09877}, 2018.

\bibitem[Chen et~al.(2017)Chen, Su, and Xu]{Chen_2017}
Yudong Chen, Lili Su, and Jiaming Xu.
\newblock Distributed statistical machine learning in adversarial settings.
\newblock \emph{Proceedings of the ACM on Measurement and Analysis of Computing
  Systems}, 1\penalty0 (2):\penalty0 1–25, Dec 2017.
\newblock ISSN 2476-1249.
\newblock \doi{10.1145/3154503}.

\bibitem[Cohen et~al.(2016)Cohen, Lee, Miller, Pachocki, and
  Sidford]{cohen2016geometric}
Michael~B Cohen, Yin~Tat Lee, Gary Miller, Jakub Pachocki, and Aaron Sidford.
\newblock Geometric median in nearly linear time.
\newblock In \emph{Proceedings of the forty-eighth annual ACM symposium on
  Theory of Computing}, pp.\  9--21, 2016.

\bibitem[Damaskinos et~al.(2019)Damaskinos, El~Mhamdi, Guerraoui, Guirguis, and
  Rouault]{Damaskinos:265684}
Georgios Damaskinos, El~Mahdi El~Mhamdi, Rachid Guerraoui, Arsany
  Hany~Abdelmessih Guirguis, and Sébastien Louis~Alexandre Rouault.
\newblock Aggregathor: Byzantine machine learning via robust gradient
  aggregation.
\newblock \emph{Conference on Systems and Machine Learning (SysML) 2019,
  Stanford, CA, USA}, pp.\ ~19, 2019.

\bibitem[Data \& Diggavi(2020)Data and Diggavi]{data2020byzantine}
Deepesh Data and Suhas Diggavi.
\newblock Byzantine-resilient sgd in high dimensions on heterogeneous data.
\newblock \emph{arXiv preprint arXiv:2005.07866}, 2020.

\bibitem[Data \& Diggavi(2021)Data and Diggavi]{data2020byzantineFL}
Deepesh Data and Suhas Diggavi.
\newblock Byzantine-resilient high-dimensional sgd with local iterations on
  heterogeneous data.
\newblock In \emph{ICML 2021 - 37th International Conference on Machine
  Learning}, 2021.

\bibitem[El~Mhamdi et~al.(2018)El~Mhamdi, Guerraoui, and
  Rouault]{mhamdi2018hidden}
El~Mahdi El~Mhamdi, Rachid Guerraoui, and S{\'e}bastien Rouault.
\newblock The hidden vulnerability of distributed learning in byzantium.
\newblock In \emph{International Conference on Machine Learning}, pp.\
  3521--3530. PMLR, 2018.

\bibitem[El~Mhamdi et~al.(2021{\natexlab{a}})El~Mhamdi, Farhadkhani, Guerraoui,
  Guirguis, Hoang, and Rouault]{el2021collaborative}
El~Mahdi El~Mhamdi, Sadegh Farhadkhani, Rachid Guerraoui, Arsany Guirguis,
  L{\^e}-Nguy{\^e}n Hoang, and S{\'e}bastien Rouault.
\newblock Collaborative learning in the jungle (decentralized, byzantine,
  heterogeneous, asynchronous and nonconvex learning).
\newblock \emph{Advances in Neural Information Processing Systems}, 34,
  2021{\natexlab{a}}.

\bibitem[El~Mhamdi et~al.(2021{\natexlab{b}})El~Mhamdi, Guerraoui, and
  Rouault]{mhamdi2021momentum}
El~Mahdi El~Mhamdi, Rachid Guerraoui, and S{\'e}bastien Rouault.
\newblock Distributed momentum for byzantine-resilient stochastic gradient
  descent.
\newblock In \emph{ICLR 2021 - International Conference on Learning
  Representations}, 2021{\natexlab{b}}.

\bibitem[Ghosh et~al.(2019)Ghosh, Hong, Yin, and Ramchandran]{ghosh2019robust}
Avishek Ghosh, Justin Hong, Dong Yin, and Kannan Ramchandran.
\newblock Robust federated learning in a heterogeneous environment.
\newblock \emph{arXiv preprint arXiv:1906.06629}, 2019.

\bibitem[Gupta \& Vaidya(2020)Gupta and Vaidya]{gupta2020resilience}
Nirupam Gupta and Nitin~H Vaidya.
\newblock Resilience in collaborative optimization: redundant and independent
  cost functions.
\newblock \emph{arXiv preprint arXiv:2003.09675}, 2020.

\bibitem[He et~al.(2022)He, Karimireddy, and Jaggi]{he2022byzantine}
Lie He, Sai~Praneeth Karimireddy, and Martin Jaggi.
\newblock Byzantine-robust decentralized learning via self-centered clipping.
\newblock \emph{arXiv preprint arXiv:2202.01545}, 2022.

\bibitem[Kairouz et~al.(2019)Kairouz, McMahan, Avent, Bellet, Bennis, Bhagoji,
  Bonawitz, Charles, Cormode, Cummings, D'Oliveira, Rouayheb, Evans, Gardner,
  Garrett, Gasc{\'{o}}n, Ghazi, Gibbons, Gruteser, Harchaoui, He, He, Huo,
  Hutchinson, Hsu, Jaggi, Javidi, Joshi, Khodak, Konecn{\'{y}}, Korolova,
  Koushanfar, Koyejo, Lepoint, Liu, Mittal, Mohri, Nock, {\"{O}}zg{\"{u}}r,
  Pagh, Raykova, Qi, Ramage, Raskar, Song, Song, Stich, Sun, Suresh,
  Tram{\`{e}}r, Vepakomma, Wang, Xiong, Xu, Yang, Yu, Yu, and
  Zhao]{kairouz2019federated}
Peter Kairouz, H.~Brendan McMahan, Brendan Avent, Aur{\'{e}}lien Bellet, Mehdi
  Bennis, Arjun~Nitin Bhagoji, Keith Bonawitz, Zachary Charles, Graham Cormode,
  Rachel Cummings, Rafael G.~L. D'Oliveira, Salim~El Rouayheb, David Evans,
  Josh Gardner, Zachary Garrett, Adri{\`{a}} Gasc{\'{o}}n, Badih Ghazi,
  Phillip~B. Gibbons, Marco Gruteser, Za{\"{\i}}d Harchaoui, Chaoyang He, Lie
  He, Zhouyuan Huo, Ben Hutchinson, Justin Hsu, Martin Jaggi, Tara Javidi,
  Gauri Joshi, Mikhail Khodak, Jakub Konecn{\'{y}}, Aleksandra Korolova,
  Farinaz Koushanfar, Sanmi Koyejo, Tancr{\`{e}}de Lepoint, Yang Liu, Prateek
  Mittal, Mehryar Mohri, Richard Nock, Ayfer {\"{O}}zg{\"{u}}r, Rasmus Pagh,
  Mariana Raykova, Hang Qi, Daniel Ramage, Ramesh Raskar, Dawn Song, Weikang
  Song, Sebastian~U. Stich, Ziteng Sun, Ananda~Theertha Suresh, Florian
  Tram{\`{e}}r, Praneeth Vepakomma, Jianyu Wang, Li~Xiong, Zheng Xu, Qiang
  Yang, Felix~X. Yu, Han Yu, and Sen Zhao.
\newblock Advances and open problems in federated learning.
\newblock \emph{arXiv preprint arXiv:1912.04977}, 2019.

\bibitem[Karimireddy et~al.(2019)Karimireddy, Rebjock, Stich, and
  Jaggi]{karimireddy2019error}
Sai~Praneeth Karimireddy, Quentin Rebjock, Sebastian~U Stich, and Martin Jaggi.
\newblock {Error Feedback Fixes SignSGD and other Gradient Compression
  Schemes}.
\newblock In \emph{ICML 2019 - Proceedings of the 36th International Conference
  on Machine Learning}, 2019.

\bibitem[Karimireddy et~al.(2020)Karimireddy, Kale, Mohri, Reddi, Stich, and
  Suresh]{karimireddy2020scaffold}
Sai~Praneeth Karimireddy, Satyen Kale, Mehryar Mohri, Sashank Reddi, Sebastian
  Stich, and Ananda~Theertha Suresh.
\newblock {SCAFFOLD}: Stochastic controlled averaging for federated learning.
\newblock In \emph{ICML 2020 - International Conference on Machine Learning},
  2020.

\bibitem[Karimireddy et~al.(2021)Karimireddy, He, and
  Jaggi]{karimireddy2020learning}
Sai~Praneeth Karimireddy, Li~He, and Martin Jaggi.
\newblock Learning from history for byzantine robust optimization.
\newblock In \emph{ICML 2021 - 37th International Conference on Machine
  Learning}, 2021.

\bibitem[Khaled et~al.(2020)Khaled, Mishchenko, and
  Richt{\'a}rik]{khaled2020tighter}
Ahmed Khaled, Konstantin Mishchenko, and Peter Richt{\'a}rik.
\newblock Tighter theory for local sgd on identical and heterogeneous data.
\newblock In \emph{AISTATS 2020 - International Conference on Artificial
  Intelligence and Statistics}, 2020.

\bibitem[Lamport et~al.(2019)Lamport, Shostak, and Pease]{lamport2019byzantine}
Leslie Lamport, Robert Shostak, and Marshall Pease.
\newblock The byzantine generals problem.
\newblock In \emph{Concurrency: the Works of Leslie Lamport}, pp.\  203--226.
  2019.

\bibitem[LeCun et~al.(1998)LeCun, Bottou, Bengio, and
  Haffner]{lecun1998gradient}
Yann LeCun, L{\'e}on Bottou, Yoshua Bengio, and Patrick Haffner.
\newblock Gradient-based learning applied to document recognition.
\newblock \emph{Proceedings of the IEEE}, 86\penalty0 (11):\penalty0
  2278--2324, 1998.

\bibitem[Li et~al.(2019)Li, Xu, Chen, Giannakis, and Ling]{li2019rsa}
Liping Li, Wei Xu, Tianyi Chen, Georgios~B Giannakis, and Qing Ling.
\newblock {RSA: Byzantine-robust stochastic aggregation methods for distributed
  learning from heterogeneous datasets}.
\newblock In \emph{Proceedings of the AAAI Conference on Artificial
  Intelligence}, volume~33, pp.\  1544--1551, 2019.

\bibitem[Ma et~al.(2018)Ma, Bassily, and Belkin]{ma2018power}
Siyuan Ma, Raef Bassily, and Mikhail Belkin.
\newblock The power of interpolation: Understanding the effectiveness of sgd in
  modern over-parametrized learning.
\newblock In \emph{International Conference on Machine Learning}, pp.\
  3325--3334. PMLR, 2018.

\bibitem[McMahan et~al.(2016)McMahan, Moore, Ramage, Hampson, and
  y~Arcas]{mcmahan2016communicationefficient}
H.~Brendan McMahan, Eider Moore, Daniel Ramage, Seth Hampson, and
  Blaise~Agüera y~Arcas.
\newblock Communication-efficient learning of deep networks from decentralized
  data.
\newblock \emph{arXiv preprint arXiv:1602.05629}, 2016.

\bibitem[Meeds et~al.(2015)Meeds, Hendriks, Al~Faraby, Bruntink, and
  Welling]{Meeds_2015}
Edward Meeds, Remco Hendriks, Said Al~Faraby, Magiel Bruntink, and Max Welling.
\newblock Mlitb: machine learning in the browser.
\newblock \emph{PeerJ Computer Science}, 1:\penalty0 e11, Jul 2015.
\newblock ISSN 2376-5992.
\newblock \doi{10.7717/peerj-cs.11}.
\newblock URL \url{http://dx.doi.org/10.7717/peerj-cs.11}.

\bibitem[Meng et~al.(2020)Meng, Vaswani, Laradji, Schmidt, and
  Lacoste-Julien]{meng2020fast}
Si~Yi Meng, Sharan Vaswani, Issam~Hadj Laradji, Mark Schmidt, and Simon
  Lacoste-Julien.
\newblock Fast and furious convergence: Stochastic second order methods under
  interpolation.
\newblock In \emph{International Conference on Artificial Intelligence and
  Statistics}, pp.\  1375--1386. PMLR, 2020.

\bibitem[Minsker et~al.(2015)]{minsker2015geometric}
Stanislav Minsker et~al.
\newblock Geometric median and robust estimation in banach spaces.
\newblock \emph{Bernoulli}, 21\penalty0 (4):\penalty0 2308--2335, 2015.

\bibitem[Miura \& Harada(2015)Miura and Harada]{miura2015implementation}
Ken Miura and Tatsuya Harada.
\newblock Implementation of a practical distributed calculation system with
  browsers and javascript, and application to distributed deep learning.
\newblock \emph{arXiv preprint arXiv:1503.05743}, 2015.

\bibitem[Oja(1982)]{oja1982simplified}
Erkki Oja.
\newblock Simplified neuron model as a principal component analyzer.
\newblock \emph{Journal of mathematical biology}, 15\penalty0 (3):\penalty0
  267--273, 1982.

\bibitem[Paszke et~al.(2019)Paszke, Gross, Massa, Lerer, Bradbury, Chanan,
  Killeen, Lin, Gimelshein, Antiga, et~al.]{paszke2019pytorch}
Adam Paszke, Sam Gross, Francisco Massa, Adam Lerer, James Bradbury, Gregory
  Chanan, Trevor Killeen, Zeming Lin, Natalia Gimelshein, Luca Antiga, et~al.
\newblock Pytorch: An imperative style, high-performance deep learning library.
\newblock In \emph{Advances in Neural Information Processing Systems}, pp.\
  8024--8035, 2019.

\bibitem[Pillutla et~al.(2019)Pillutla, Kakade, and
  Harchaoui]{pillutla2019robust}
Krishna Pillutla, Sham~M. Kakade, and Zaid Harchaoui.
\newblock {Robust Aggregation for Federated Learning}.
\newblock \emph{arXiv preprint arXiv:1912.13445}, 2019.

\bibitem[Rajput et~al.(2019)Rajput, Wang, Charles, and
  Papailiopoulos]{rajput2019detox}
Shashank Rajput, Hongyi Wang, Zachary Charles, and Dimitris Papailiopoulos.
\newblock Detox: A redundancy-based framework for faster and more robust
  gradient aggregation.
\newblock \emph{arXiv preprint arXiv:1907.12205}, 2019.

\bibitem[{Sattler} et~al.(2020){Sattler}, {Müller}, {Wiegand}, and
  {Samek}]{9054676}
F.~{Sattler}, K.~{Müller}, T.~{Wiegand}, and W.~{Samek}.
\newblock On the byzantine robustness of clustered federated learning.
\newblock In \emph{ICASSP 2020 - 2020 IEEE International Conference on
  Acoustics, Speech and Signal Processing (ICASSP)}, pp.\  8861--8865, 2020.

\bibitem[Schmidt \& Roux(2013)Schmidt and Roux]{schmidt2013fast}
Mark Schmidt and Nicolas~Le Roux.
\newblock Fast convergence of stochastic gradient descent under a strong growth
  condition.
\newblock \emph{arXiv preprint arXiv:1308.6370}, 2013.

\bibitem[Sun et~al.(2019)Sun, Kairouz, Suresh, and McMahan]{sun2019can}
Ziteng Sun, Peter Kairouz, Ananda~Theertha Suresh, and H~Brendan McMahan.
\newblock Can you really backdoor federated learning?
\newblock \emph{arXiv preprint arXiv:1911.07963}, 2019.

\bibitem[Vaswani et~al.(2019{\natexlab{a}})Vaswani, Bach, and
  Schmidt]{vaswani2019fast}
Sharan Vaswani, Francis Bach, and Mark Schmidt.
\newblock Fast and faster convergence of sgd for over-parameterized models and
  an accelerated perceptron.
\newblock In \emph{AISTATS 2019 - The 22nd International Conference on
  Artificial Intelligence and Statistics}, 2019{\natexlab{a}}.

\bibitem[Vaswani et~al.(2019{\natexlab{b}})Vaswani, Mishkin, Laradji, Schmidt,
  Gidel, and Lacoste-Julien]{vaswani2019painless}
Sharan Vaswani, Aaron Mishkin, Issam Laradji, Mark Schmidt, Gauthier Gidel, and
  Simon Lacoste-Julien.
\newblock Painless stochastic gradient: Interpolation, line-search, and
  convergence rates.
\newblock \emph{Advances in neural information processing systems}, 32,
  2019{\natexlab{b}}.

\bibitem[Wang et~al.(2020)Wang, Sreenivasan, Rajput, Vishwakarma, Agarwal,
  Sohn, Lee, and Papailiopoulos]{wang2020attack}
Hongyi Wang, Kartik Sreenivasan, Shashank Rajput, Harit Vishwakarma, Saurabh
  Agarwal, Jy-yong Sohn, Kangwook Lee, and Dimitris Papailiopoulos.
\newblock Attack of the tails: Yes, you really can backdoor federated learning.
\newblock \emph{arXiv preprint arXiv:2007.05084}, 2020.

\bibitem[Wu et~al.(2020)Wu, Ling, Chen, and Giannakis]{wu2020federated}
Zhaoxian Wu, Qing Ling, Tianyi Chen, and Georgios~B Giannakis.
\newblock Federated variance-reduced stochastic gradient descent with
  robustness to byzantine attacks.
\newblock \emph{IEEE Transactions on Signal Processing}, 68:\penalty0
  4583--4596, 2020.

\bibitem[Xie et~al.(2019)Xie, Koyejo, and Gupta]{xie2018zeno}
Cong Xie, Oluwasanmi Koyejo, and Indranil Gupta.
\newblock Zeno: Distributed stochastic gradient descent with suspicion-based
  fault-tolerance.
\newblock In \emph{ICML 2019 - 35th International Conference on Machine
  Learning}, 2019.

\bibitem[Xie et~al.(2020)Xie, Koyejo, and Gupta]{xie2019fall}
Cong Xie, Oluwasanmi Koyejo, and Indranil Gupta.
\newblock {Fall of Empires: Breaking Byzantine-tolerant SGD by Inner Product
  Manipulation}.
\newblock In \emph{UAI - Proceedings of The 35th Uncertainty in Artificial
  Intelligence Conference}, 2020.

\bibitem[Yang \& Li(2021)Yang and Li]{pmlr-v139-yang21e}
Yi-Rui Yang and Wu-Jun Li.
\newblock Basgd: Buffered asynchronous sgd for byzantine learning.
\newblock In Marina Meila and Tong Zhang (eds.), \emph{Proceedings of the 38th
  International Conference on Machine Learning}, volume 139 of
  \emph{Proceedings of Machine Learning Research}, pp.\  11751--11761. PMLR,
  18--24 Jul 2021.
\newblock URL \url{https://proceedings.mlr.press/v139/yang21e.html}.

\bibitem[Yin et~al.(2018)Yin, Chen, Ramchandran, and
  Bartlett]{yin2018byzantinerobust}
Dong Yin, Yudong Chen, Kannan Ramchandran, and Peter Bartlett.
\newblock Byzantine-robust distributed learning: Towards optimal statistical
  rates.
\newblock \emph{arXiv preprint arXiv:1803.01498}, 2018.

\bibitem[Yu et~al.(2019)Yu, Yang, and Zhu]{yu2019parallel}
Hao Yu, Sen Yang, and Shenghuo Zhu.
\newblock Parallel restarted sgd with faster convergence and less
  communication: Demystifying why model averaging works for deep learning.
\newblock In \emph{AAAI 2019 - Conference on Artificial Intelligence}, 2019.

\end{thebibliography}
\bibliographystyle{iclr2022_conference}

\clearpage
\appendix
\appendix

\onecolumn

\clearpage
\section{Experiment setup and additional experiments} \label{sec:additional_exps}

\subsection{Experiment setup} \label{ssec:setups}

\subsubsection{General setup}
The default experiment setup is listed in \Cref{tab:setup:mnist:default}.
\begin{table}[h]
    \caption{Default experimental settings for MNIST}
    \centering
    \small
    \label{tab:setup:mnist:default}%
    \begin{tabular}{ll}
        \toprule
        Dataset              & MNIST                                           \\
        Architecture         & CONV-CONV-DROPOUT-FC-DROPOUT-FC                 \\
        Training objective   & Negative log likelihood loss                    \\
        Evaluation objective & Top-1 accuracy                                  \\
        \midrule
        Batch size           & $32 \times \text{number of workers}$            \\
        Momentum             & 0 or 0.9                                        \\
        Learning rate        & 0.01                                            \\
        LR decay             & No                                              \\
        LR warmup            & No                                              \\
        \# Iterations        & 600 or 4500                                     \\
        Weight decay         & No                                              \\
        \midrule
        Repetitions          & 3, with varying seeds                           \\
        Reported metric      & Mean test accuracy over the last 150 iterations \\
        \bottomrule
    \end{tabular}
\end{table}

By default the hyperparameters of the aggregators are summarized as follows

    {
        \hfill
        \begin{tabular}{ll}
            \toprule
            Aggregators & Hyperparameters           \\
            \midrule
            \krum       & N/A                       \\
            \cm         & N/A                       \\
            \rfa        & $T=8$                     \\
            \tm         & $b=q$                     \\
            \cclip      & $\tau=\frac{10}{1-\beta}$ \\
            \bottomrule
        \end{tabular}
        \hfill
    }

\subsubsection{Constructing datasets}
The MNIST dataset has 10 classes each with similar amount of samples. In this part, we discuss how to process and distribute MNIST to each workers in order to achieve long-tailness and heterogeneity.

\paragraph{Long-tailness.} The long-tailness (*-LT) is achieved by sampling class with exponentially decreasing portions $\gamma\in(0, 1]$. That is, for class $i\in[10]$, we only randomly sample $\gamma^i$ portion of all samples in class $i$. We define $\alpha$ as the ratio of the largest class over the smallest class, which can be written as $\alpha=\tfrac{1}{\gamma^9}$.
    For example, if $\gamma=1$, then all classes have same amount of samples and thus $\alpha=1$; if $\gamma=0.5$ then $\alpha=2^9 = 512$.
    Note that the same procedure has to be applied to the test dataset.

    \paragraph{Heterogeneity.} Steps to construct IID/\noniid dataset from MNIST dataset
    \begin{enumerate}[nosep]
            \item Sort the training dataset by its labels.
            \item Evenly divide the sorted training dataset into chunks of same size. The number of chunks equals the number of good workers. If the last chunk has fewer samples, we augment it with samples from itself.
            \item Shuffle the samples within the same worker.
        \end{enumerate}

    \paragraph{Heterogeneity + Long-tailness.} First transform the training dataset into long-tail dataset, then feed it to the previous procedure to introduce heterogeneity.

    \paragraph{About dataset on Byzantine workers.} The training set is divided by the number of good workers. So the good workers has to full information of training dataset. The Byzantine worker has access to the whole training dataset.

    \subsubsection{Setup for each experiment}
    In \Cref{tab:setups:mnist:all}, we list the hyperparameters for the experiments.
    \begin{table}[!h]
            \caption{Setups for each experiment.}
            \centering
            \label{tab:setups:mnist:all}%
            \begin{tabular}{lllll ll}
                \toprule
                                                 & n  & q & momentum             & Iters & LT                       & NonIID         \\\midrule
                \Cref{tab:noattack:naive}        & 24 & 0 & 0                    & 4500  & $\alpha=1$, $\alpha=500$ & \iid / \noniid \\
                \Cref{tab:mimic_attack:naive}    & 25 & 5 & 0                    & 600   & $\alpha=1$   (balanced)  & \iid / \noniid \\
                \Cref{tab:noattack:resample}     & 24 & 0 & 0                    & 4500  & $\alpha=1$, $\alpha=500$ & \iid / \noniid \\
                \Cref{tab:mimic_attack:resample} & 25 & 5 & 0                    & 600   & $\alpha=1$ (balanced)    & \iid / \noniid \\
                \Cref{fig:main}                  & 25 & 5 & 0 / 0.9              & 600   & $\alpha=1$ (balanced)    & \noniid        \\
                \Cref{fig:hp}                    & 53 & 5 & 0 / 0.9              & 600   & $\alpha=1$ (balanced)    & \noniid        \\
                \Cref{fig:grad-size}             & 25 & 5 & 0 / 0.5 / 0.9 / 0.99 & 600   & $\alpha=1$ (balanced)    & \noniid        \\
                \Cref{fig:cp-scaling-tau10}      & 25 & 5 & 0 / 0.5 / 0.9 / 0.99 & 1200  & $\alpha=1$ (balanced)    & \noniid        \\
                \Cref{fig:krum-selections}       & 20 & 3 & 0                    & 1200  & $\alpha=1$ (balanced)    & \noniid        \\
                \Cref{fig:overparameterization}  & 20 & 3 & 0                    & 3000  & $\alpha=1$ (balanced)    & \noniid        \\
                \Cref{fig:bucketing}             & 24 & 3 & 0                    & 1200  & $\alpha=1$ (balanced)    & \noniid        \\
                \bottomrule
            \end{tabular}
            \hfill
        \end{table}

    \paragraph{IPM Attack in \Cref{fig:main} and \Cref{fig:hp}.} We set the strength of the attack $\epsilon=0.1$.

    \paragraph{ALIE Attack in in \Cref{fig:main}.} The hyperparameter $z$ for ALIE is computed according to  \citep{baruch2019little}
    $$z=\max_{z} \left(\phi(z) < \frac{n-q-s}{n-q} \right)$$
    where $s=\lfloor \frac{n}{2}+1\rfloor-q$ and $\phi$ is the cumulative standard normal function. In our setup, the $z\approx 0.25$.

    \subsubsection{Running environment}
    We summarize the running environment of this paper as in \Cref{tab:runtime}.
    \begin{table}[!h]
            \caption{Runtime hardwares and softwares.}
            \centering
            \label{tab:runtime}%
            \begin{tabular}{p{0.22\linewidth} p{0.7\linewidth}}
                \toprule
                CPU                       &                                             \\
                \hspace{1em} Model name   & Intel (R) Xeon (R) Gold 6132 CPU @ 2.60 GHz \\
                \hspace{1em} \# CPU(s)    & 56                                          \\
                \hspace{1em} NUMA node(s) & 2                                           \\ \midrule
                GPU                       &                                             \\
                \hspace{1em} Product Name & Tesla V100-SXM2-32GB                        \\
                \hspace{1em} CUDA Version & 11.0                                        \\ \midrule
                PyTorch                   &                                             \\
                \hspace{1em} Version      & 1.7.1                                       \\
                \bottomrule
            \end{tabular}
        \end{table}



    \subsection{Additional experiments}\label{ssec:additional_exps}
    \subsubsection{Clipping radius scaling}\label{sssec:clipping_radius}
    The radius $\tau$ of \cclip depends on the norm of good gradients. However, PyTorch implements SGD with momentum using the following formula
    \begin{align*}
            \mm_i^t = \beta \mm_i^{t-1} + \gg_i(\xx^{t-1}) \quad \text{ for every } i \in \cG
        \end{align*}
    which may leads to the increase in the gradient norm.

    \paragraph{Gradient norms.} In \Cref{fig:grad-size} we present the averaged gradient norm from all good workers. Here we use \cclip as the aggregator and $\tau=\frac{10}{1-\beta}$. The norm of gradients are computed before aggregation. Even though the dataset on workers are \noniid, the gradient norms are roughly of same order. The gradient dissimilarity $\zeta^2$ also increases accordingly.
    \begin{figure}[H]
            \vspace{-3mm}
            \centering
            \includegraphics[
                width=\linewidth
            ]{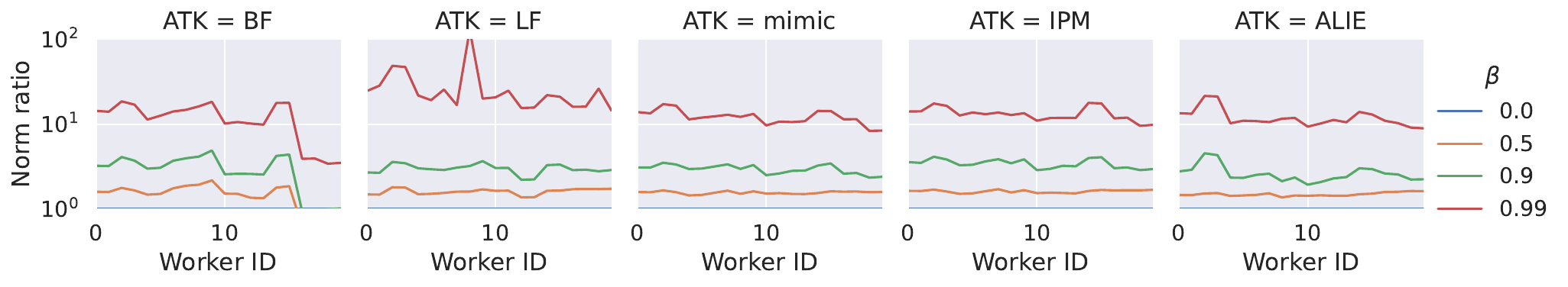}
            \caption{
                The ratio of norm of good gradients with momentum $\beta$ over no momentum under different attacks.
            }
            \vspace{-2mm}
            \label{fig:grad-size}
        \end{figure}

    \paragraph{Scaled clipping radius.} As the gradient norm increases with momentum $\beta$, the clipping radius should increase accordingly. In \Cref{fig:cp-scaling-tau10} we compare 3 schemes: 1) no scaling ($\tau=10$, $\beta=0$); 2) \textit{linear} scaling $\frac{10}{1-\beta}$; 3) \textit{sqrt} scaling $\frac{10}{\sqrt{1-\beta}}$. The no scaling scheme convergences but slower while with momentum. The linear scaling is usually better than \emph{sqrt} scaling and with bucketing it becomes more stable. However, The scaled clipping radius fails for $\beta=0.99$ under label flipping attack. This is because the gradient can be very large and $\zeta^2$ dominates. So in general, a linear scaling of clipping radius with momentum $\beta=0.9$ would be a good choice.
    \begin{figure}[H]
        \vspace{-3mm}
        \centering
        \includegraphics[
            width=\linewidth
        ]{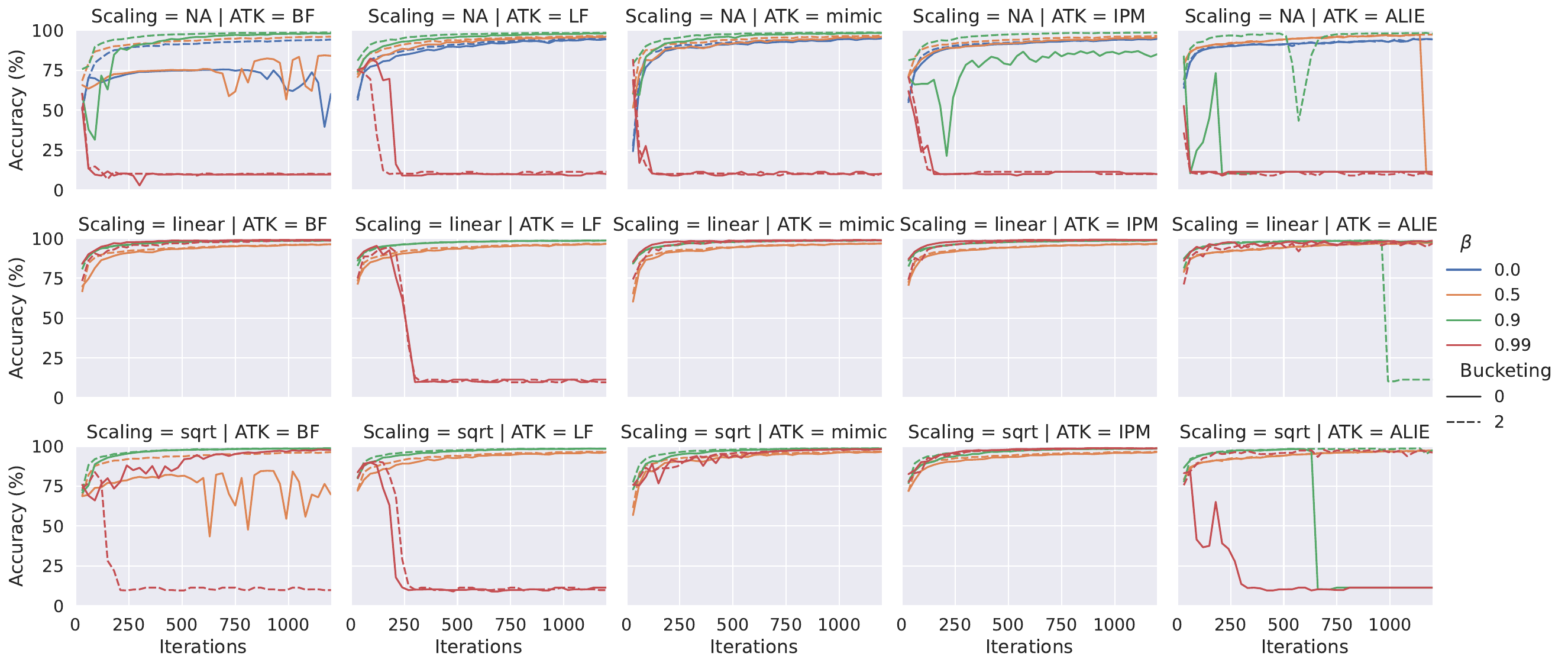}
        \caption{
            Convergence of \cclip with $\tau=10, \frac{10}{1-\beta}, \frac{10}{\sqrt{1-\beta}}$ for $\beta=0, 0.5, 0.9, 0.99$. The $s$ is the bucketing hyperparameter.
        }
        \vspace{-2mm}
        \label{fig:cp-scaling-tau10}
    \end{figure}

    \subsubsection{Demonstration of effects of bucketing through the selections of \krum}
    In the main text we have theoretically show that bucketing helps aggregators alleviate the impact of \noniid. In this section we empirically show that after bucketing aggregators can incorporate updates more evenly from good workers and therefore the problem of \noniid among good workers is less significant.  Since \krum outputs the id of the selected device, it is very convenient to record the frequency of each worker being selected. Since bucketing replicates each worker for $s$ times, we divide their frequencies by $s$ for normalization. From \Cref{fig:krum-selections}, we can see that without bucketing \krum basically almost always selects updates from Byzantine workers while with larger $s$, the selection becomes more evenly distributed.

    \begin{figure}[H]
        \vspace{-3mm}
        \centering
        \includegraphics[
            width=0.9\linewidth
        ]{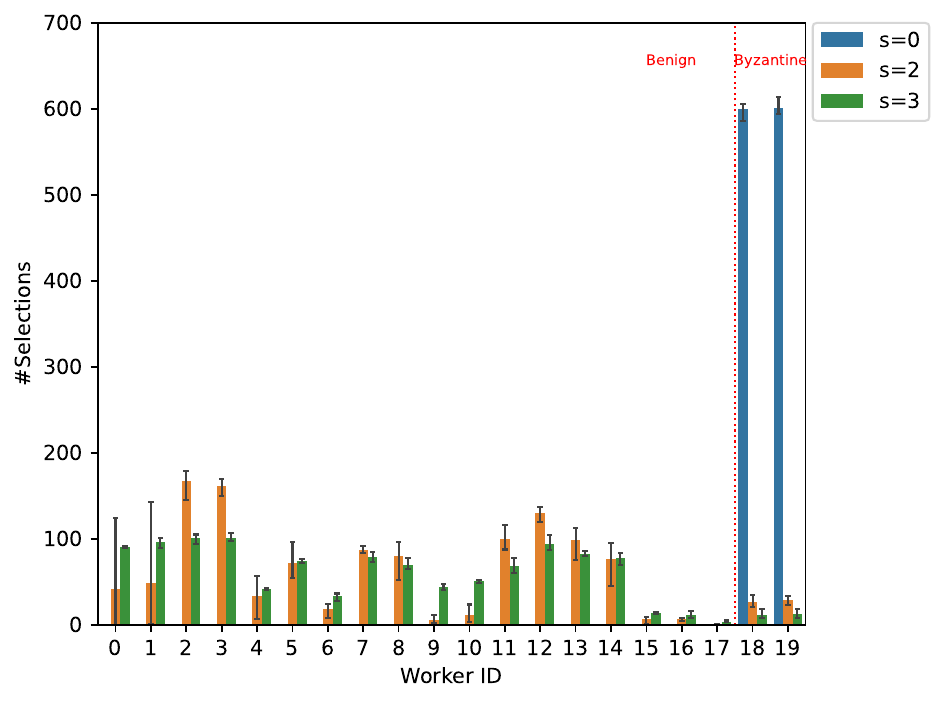}
        \caption{
            The selected workers of \krum for bucketing coefficient $s=0,2,3$. There are 20 workers and the last 2 workers (worker id=18,19) are Byzantine with label-flipping attack.
        }
        \vspace{-2mm}
        \label{fig:krum-selections}
    \end{figure}

    \subsubsection{Overparameterization}\label{sssec:overparameterization}
    The architecture of the neural net used in the experiments can be scaled to make it overparameterized. We add more parameters to the model by multiplying the channels of 2D \verb|Conv| layer and fully connected layer by a factor of `scale'. So the original model has a scale of 1. We show the training losses decrease faster for overparameterized models in \Cref{fig:overparameterization}. As we can see, the convergence behaviors are similar for different model scales with overparameterized models having smaller training loss despite the existence of Byzantine workers.
    \begin{figure}[H]
        \vspace{-3mm}
        \centering
        \includegraphics[
            width=\linewidth
        ]{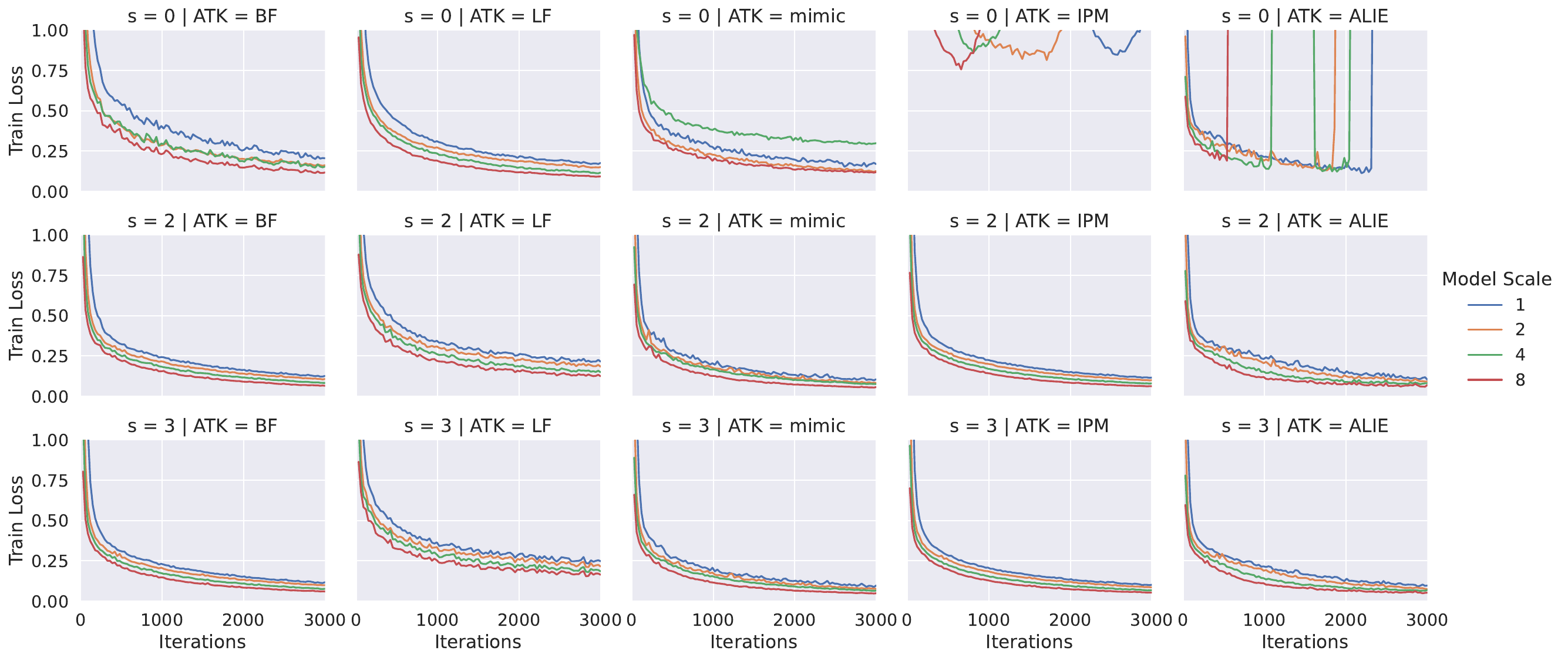}
        \caption{
            The training loss of models of different levels of overparameterization.
        }
        \vspace{-2mm}
        \label{fig:overparameterization}
    \end{figure}
    In \Cref{fig:overparameterization:B2}, we explicitly investigate the influence of overparameterization on $B^2$ defined in \eqref{eqn:asm-overpara}. As we can see, heterogeneity bound $B^2$ decreases with increasing level of overparameterization, showcasing how overparameterization minimizes the local objectives in the presence of Byzantine workers. It supports our theory in \Cref{ssec:overparameterization} that overparameterization can fix the convergence, making it possible to achieve practical Byzantine-robust learning.
    The underlying base aggregator is \rfa.
    \begin{figure}[H]
        \vspace{-3mm}
        \centering
        \includegraphics[
            width=\linewidth
        ]{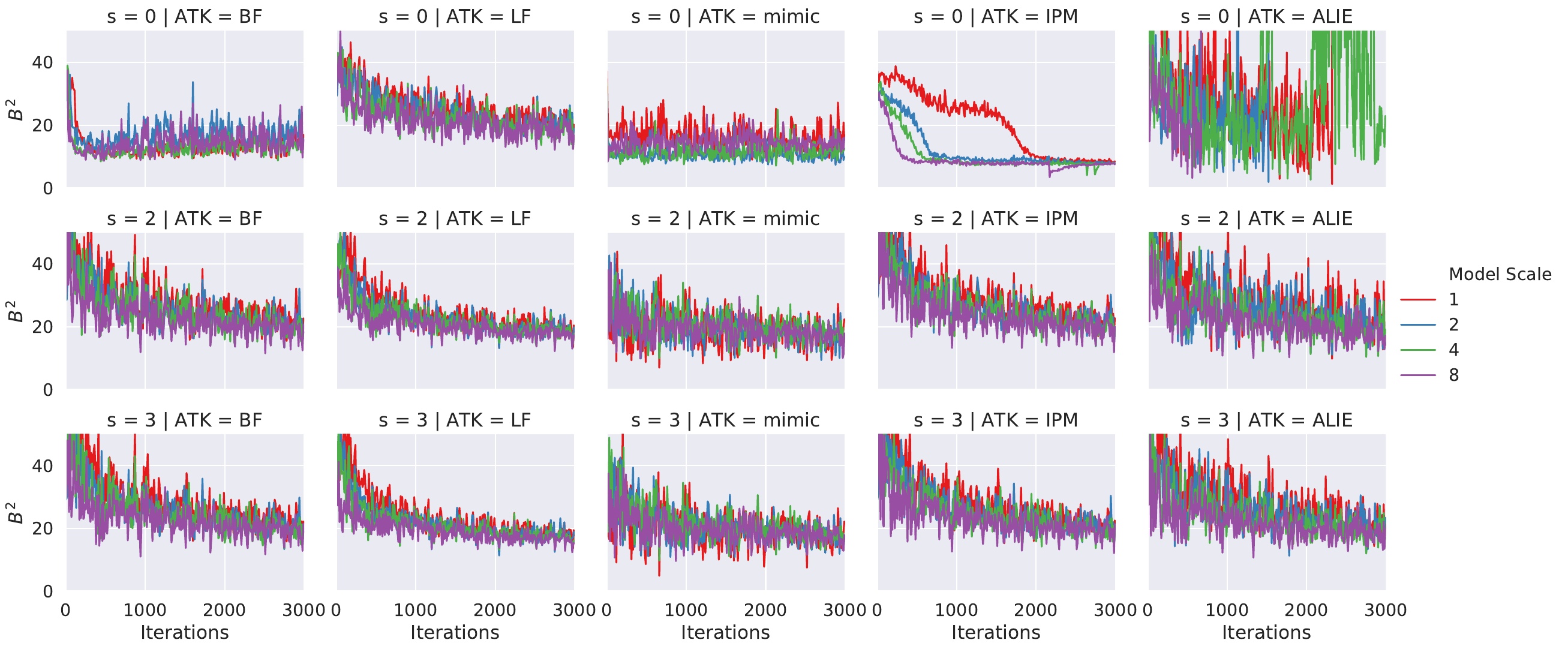}
        \caption{
            The $B^2$ in \eqref{eqn:asm-overpara} for different levels of overparameterization.
        }
        \vspace{-2mm}
        \label{fig:overparameterization:B2}
    \end{figure}


    \subsubsection{Resampling - variant of bucketing}\label{ssec:resampling_vs_bucketing}
    In the previous version of this work we repeat the gradients for $s$ times and then put $sn$ gradients into $n$ buckets.
    The results in \Cref{fig:bucketing} suggest that the convergence rate of bucketing and resampling is almost the same. So aggregators can benefit more from bucketing as it reduces the number of input gradients and therefore reduce the complexity.
    \begin{figure}[H]
        \vspace{-3mm}
        \centering
        \includegraphics[
            width=\linewidth
        ]{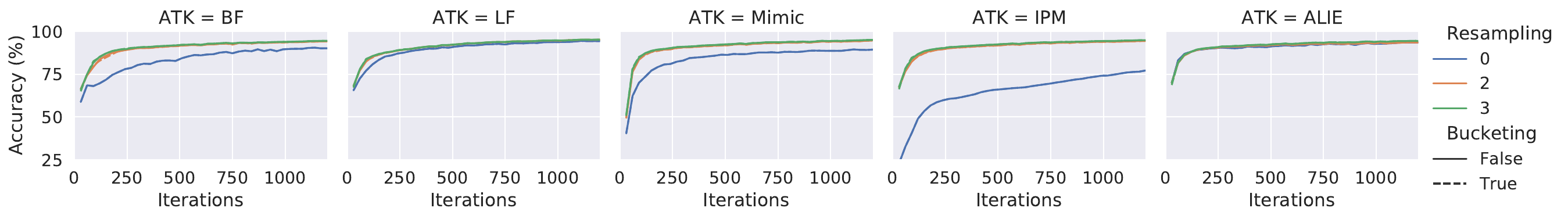}
        \caption{
            The convergence SGD with bucketing and resampling under different attacks. The underlying aggregator is \rfa.
        }
        \vspace{-2mm}
        \label{fig:bucketing}
    \end{figure}

    \section{Implementing the mimic attack}\label{sec:app-mimic}
    The \Cref{ssec:mimic} describes the idea and formulation of the mimic attack. In this section, we discuss how to pick $i_\star$ and implement the mimic attack efficiently. To pick $i_\star$, we use an initial phase ($\cI^0 \approx$ 1 epoch) to compute a direction $\zz$ of maximum variance of the outputs of the good workers:
    \[
        \zz = \argmax_{\norm{\zz}=1}~ \zz^\top \Big(\sum_{t \in \cI_0} \sum_{i \in \cG} (\xx_i^t - \muv) (\xx_i^t - \muv)^{\top}\Big) \zz \quad \text{ where } \quad \muv = \frac{1}{\abs{\cG} \abs{\cI_0}} \sum_{i \in \cG, t \in \cI_0} \xx_i^t\,.
    \]
    Then we pick a worker $i^\star$ to mimic by computing\vspace{-1mm}
    \[ i_\star = \argmax_{i \in \cG} \Big|\sum_{t \in \cI_0}\zz^\top\xx_i^t\Big| \,.\]
    In the following steps, we show how to solve the optimization problem.

    First, rewrite the mimic attack in its online version at time $t\in\cI_0$
    \[
        \zz^t = \argmax_{\norm{\zz}=1}~ h^t(\zz)
    \]
    where $\muv^t = \frac{1}{\abs{\cG} t} \sum_{\tau \le t} \sum_{i \in \cG}\xx_i^{\tau}$ and
    \[
        h^t(\zz) = \zz^\top \rbr*{\sum_{\tau \le t} \sum_{i \in \cG} (\xx_i^{\tau} - \muv^t) (\xx_i^{\tau} - \muv^t)^{\top}} \zz.
    \]
    Thus we can iteratively update $\muv^t$ by
    \[
        \muv^{t+1} = \frac{t}{1+t} \muv^t + \frac{1}{1+t} \frac{1}{\abs{\cG}} \sum_{i \in \cG}\xx_i^{t+1},
    \]
    and then
    \begin{align*}
        \argmax_{\norm{\zz}=1} h^{t+1}(\zz) \approx
                & \frac{t}{1+t}\zz^{t} + \frac{1}{1+t} \argmax_{\norm{\zz}=1} \zz^\top \rbr*{\sum_{i \in \cG} (\xx_i^{t+1} - \muv^{t+1}) (\xx_i^{t+1} - \muv^{t+1})^{\top}} \zz \\
        \approx &
        \frac{t}{1+t}\zz^{t} + \frac{1}{1+t} \rbr*{\sum_{i \in \cG} (\xx_i^{t+1} - \muv^{t+1}) (\xx_i^{t+1} - \muv^{t+1})^{\top}} \zz^{t}.
    \end{align*}
    The above algorithm corresponds to Oja's method for computing the top eigenvector in a streaming fashion \citep{oja1982simplified}.
    Then, in each subsequent iteration $t$, we pick
    \[ i_\star^t = \argmax_{i \in \cG} \zz^\top\xx_i^t \,.\]

    \begin{figure}
        \centering
        \begin{subfigure}{.3\columnwidth}
            \centering
            \includegraphics[width=\linewidth]{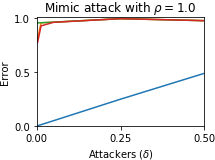}
        \end{subfigure}
        \begin{subfigure}{.3\columnwidth}
            \centering
            \includegraphics[width=\linewidth]{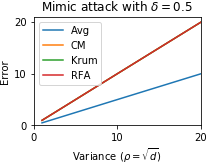}
        \end{subfigure}
        \captionsetup{font=small}

        \caption{Error with random vectors with variance $\rho^2 = d$ and $\delta$ fraction of Byzantine workers imitating a fixed good worker (say worker $1 \in \cG$). \rfa performs slightly better than \cm and \krum, but all have \emph{higher error} than simply averaging across various settings of $\delta$ and $\rho$.}
        \label{fig:demo-mimic}\vspace{-3mm}
    \end{figure}

    \textbf{Example.} Each of the good workers $i \in \cG \subseteq [n]$ has an input a $\xx_i \in \{\pm 1\}^d$ where each coordinate is an independent Rademacher random variable. The inputs then have mean $\0$ and variance $\E\norm{\xx_i}^2 = \rho^2 = d$. Now, the Byzantine attackers $j \in \cB$ have dual goals: i) escape detection, and ii) increase data imbalance. For this, we propose the following simple passive attack: pick some fixed worker $i_\star \in \cG$ (say 1) and every Byzantine worker $j \in \cB$ outputs $\xx_j = \xx_1$. The attackers cannot be filtered as they imitate an existing good worker, but still can cause imbalance in the data distribution.
    This serves as the intuition for our attack.

    \section{Constructing a robust aggregator using bucketing}

    \subsection{Supporting lemmas}
    We first start with proving the main bucketing \Cref{lemma:bucketing} restated below.

    \begin{replemma}{lemma:bucketing}
        Suppose we are given $n$ independent (but not identical) random vectors $\{\xx_1, \dots, \xx_n\}$ such that a good subset $\cG \subseteq [n]$ of size at least $\abs{\cG} \geq n(1 - \delta)$ satisfies:
        \[
            \E\norm{\xx_i - \xx_j}^2 \leq \rho^2\,, \quad \text{ for any fixed } i, j \in \cG\,.
        \]
        Define $\bar\xx := \frac{1}{\abs{\cG}} \sum_{j \in \cG} \xx_j$ and $m = \lceil n/s\rceil$.
        Let the outputs after $s$-bucketing be $\{\yy_1, \dots, \yy_{m}\}$. Then, a subset of the outputs $\tilde\cG \subseteq \{1,\ldots,m\}$ of size at least $\abs{\tilde\cG} \geq m (1 - \delta s)$ satisfies 
        \[
            \E[\yy_i] = \E[\bar\xx] \quad \text{ and } \quad \E\norm{\yy_i - \yy_j} \leq {\rho^2}/{s} \quad \text{for any fixed } i,j \in \tilde\cG\,.
        \]
    \end{replemma}\vspace{-4mm}
    \begin{proof}
        Let us define the buckets used to compute $\yy_i$ as
        \[
            B_i := \{\pi(s(i-1) + 1), \dots, \pi( \min\{s \cdot i, n\} )\} \,.
        \]

        Recall that for some permutation $\pi$ over $[n]$ and for every $i =\{1,\ldots, {m}\}$, we defined $m = \lceil n/s\rceil$ and
        \[
            \yy_i \leftarrow \frac{1}{|B_i|} \sum_{k = (i-1)\cdot s + 1}^{\min(n\,,\, i \cdot s)} \xx_{\pi(k)}  \,.
        \]
        Then, define the \emph{new} good set
        \[
            \tilde\cG = \{ i \in [m]\ |\ B_i \subseteq \cG\}
        \]
        $\tilde \cG$ contains the set of all the resampled vectors which are made up of only good vectors i.e. are uninfluenced by any Byzantine vector. Since $\abs{\cB} \leq \delta n$ and each can belong to only 1 bucket, we have that $\abs{\tilde \cG} \geq (1 - \delta s)m$. Now, for any fixed $i \in \tilde\cG$, let us look at the conditional expectation over the random permutation $\pi$ we have
        \[
            \E_{\pi}[\yy_i | i \in \tilde \cG ]= 
            \frac{1}{|B_i|} \sum_{k = (i-1)\cdot s + 1}^{\min(n\,,\, i \cdot s)} \E_{\pi} [\xx_{\pi(k)}| \pi(k)\in\cG] 
            =
            \frac{1}{\abs{\cG}} \sum_{j \in \cG} \xx_j = \bar\xx\,.
        \]
        This yields the first part of the lemma. Now we analyze the variance. 
        Thus, we can write $\yy_i = \frac{1}{s} \sum_{k\in B_i} \xx_k$. Further, $\abs{B_i} = s$ for any $i$, and $B_i \subseteq \cG$ if $i \in \tilde\cG$. With this, for any fixed $i,j \in \tilde\cG$ the variance can be written as
        \begin{align*}
            \E \norm*{\yy_i - \yy_j}^2 & = \E \norm*{\frac{1}{s} \sum_{k \in B_i}\xx_k - \frac{1}{s} \sum_{l \in B_j} \xx_l}^2                       \\
                                       &= \frac{\rho^2}{s}\,.
        \end{align*}
    \end{proof}

    This additional lemma about the maximum expected distance between good workers will also be useful later.
    \begin{lemma}[maximum good distance] \label{lem:max-dist}
        Suppose we are given the output of bucketing $\yy_1, \dots, \yy_m$ which for $m = \lceil n/s\rceil$ satisfy for any fixed $i\in \tilde\cG$, $\E[\yy_i] = \muv$ and $\E\norm{\yy_i - \muv}^2 \leq \rho^2/s$. Then, we have
        \[
            \E\sbr*{\max_{i\in \tilde\cG} \norm{\yy_i - \muv}^2} \leq n \rho^2/s^2\,.
        \]
        Further, there exist instances where
        \[
            \E\sbr*{\max_{i\in \tilde\cG} \norm{\yy_i - \muv}^2} \geq \Omega(n\rho^2/s^2)\,.
        \]
    \end{lemma}
    \begin{proof}
        For the upper bound, we simply use
        \[
            \E\sbr*{\max_{i\in \tilde\cG} \norm{\yy_i - \muv}^2} \leq \sum_{i \in \tilde\cG} \E \norm{\yy_i - \muv}^2 \leq m \rho^2/s\,.
        \]
        For the lower bound, let $\tilde\cG = [m]$ and consider $\yy_i \sim \tilde\rho\sqrt{m}\text{Bern}(p = \frac{1}{m})$. This means $\yy_i$ is either 0 or $\tilde\rho \sqrt{m}$. Further, its variance is clearly bounded by $\tilde\rho^2$. Upon drawing $m$ samples, the probability of seeing at least 1 $\yy_j = \tilde\rho \sqrt{m}$ is
        \[
            1-\Pr(\yy_i = 0\ \forall i \in [m]) = 1 - (1-\tfrac{1}{m})^m \geq 1 - \nicefrac{1}{e} \geq 0.5\,.
        \]
        Thus, with probability at least 0.5 we have
        \[
            \max_{i\in [n]} \norm{\yy_i - \muv}^2 \geq m\tilde\rho^2/2 \,.
        \]
        This directly proves our lower bound by defining $\tilde\rho^2 := \rho^2/s$ and recalling that $m = \lceil n/s\rceil$.
        Note that this lemma can be tightened if we make stronger assumptions on the noise such as $\E\norm{\yy_i - \muv}^r \leq \rbr*{\nicefrac{\rho}{\sqrt{s}}}^r$ for some large $r \geq 2$. However, we focus on using standard stochastic assumptions ($r=2$) in this work.
    \end{proof}

    \subsection{Proofs of robustness}

    Let $\{\yy_1\, \dots, \yy_m\}$ be the resampled vectors with bucketing using $s = \frac{\delta_{\max}}{\delta}$. By \Cref{lemma:bucketing}, we have that there is a $\tilde \cG \subseteq [m]$ of size $\abs{\tilde \cG} > m(1 - \delta_{\max})$ which satisfies for any fixed $i,j \in \tilde \cG$
    \[
        \E\norm{\yy_i - \yy_j}^2 \leq \frac{\delta \rho^2}{\delta_{\max}} =: \tilde\rho^2\,.
    \]
    This observation will be combined with each of the algorithms to obtain robustness guarantees.

    \paragraph{Robustness of \krum.}
    We now prove that \krum when combined with bucketing is a robust aggregator. We can rewrite the output of \krum as the following for $\delta_{\max} = 1/4 - \nu$ for some arbitrarily small positive number $\nu \in (0, 1/4)$:
    \[
        \krum(\yy_1, \dots, \yy_m)    = \argmin_{\yy_i} \min_{\abs{\cS} = {3m/4}} \sum_{j\in\cS} \norm{\yy_i - \yy_j}^2\,.
    \]
    Let $\cS^\star$ and $k^\star$ be the quantities which minimize the optimization problem solved by \krum.

    The main difficulty of analyzing \krum is that even if we succeed in selecting a $k^\star \in \tilde\cG$, $k^\star$ depends on the sampling. Hence, we \textbf{cannot} claim that the error is bounded by $\tilde\rho^2$ i.e \footnote{This issue was incorrectly overlooked in the original analysis of \krum \citep{blanchard2017machine}}
    \[
        \E\norm{\yy_{k^\star} - \yy_j}^2 \nleq \tilde\rho^2 \text{ for some fixed } j\in\tilde\cG\,.
    \]
    This is because the variance is bounded by $\tilde\rho^2$ only for a \emph{fixed} i, and not a data dependent $k^\star$. Instead, we will have to rely on \Cref{lem:max-dist} that
    \[
        \E\norm{\yy_{k^\star} - \yy_j}^2  \leq \E \max_{i \in \tilde\cG}\norm{\yy_{i} - \yy_j}^2 \leq m\tilde\rho^2\,.
    \]
    \Cref{lem:max-dist} shows that this inequality is essentially tight and hence relying on it necessarily incurs an extra factor of $m$ which can be very large. Instead, we show an alternate analysis which works for a smaller breakdown point of $\delta_{\max} = 1/4$, but \emph{does not} incur the extra $m$ factor.

    For any good input $i \in \tilde\cG$, we have
    \begin{align*}
        \norm{\yy_{k^\star}  - \bar\xx}^2           & \leq  2\norm{\yy_{k^\star} - \yy_i}^2 + 2\norm{\yy_i - \bar\xx}^2    \\
        \Rightarrow 2\norm{\yy_{k^\star} - \yy_i}^2 & \geq  \norm{\yy_{k^\star}  - \bar\xx}^2  - 2\norm{\yy_i - \bar\xx}^2
        \,.
    \end{align*}
    Further, for a bad worker $j \in \tilde\cB$ we can write
    \[
        2\norm{\yy_{k^\star} - \yy_j}^2 \geq  \norm{\yy_j - \bar\xx}^2 - 2\norm{\yy_{k^\star}  - \bar\xx}^2\,.
    \]
    Combining both and summing over $\cS^\star$,
    \begin{align*}
        \sum_{i \in \cS^\star}2  \norm{\yy_{k^\star} - \yy_i}^2 & =  \sum_{i \in \tilde\cG \cap \cS^\star}2  \norm{\yy_{k^\star} - \yy_i}^2 +  \sum_{j \in \tilde\cB \cap \cS^\star}2  \norm{\yy_{k^\star} - \yy_j}^2 \\
                                                                & \geq
        \sum_{j \in \tilde\cB \cap \cS^\star} \norm{\yy_{j} - \bar\xx}^2 - 2\sum_{i \in \tilde\cG \cap \cS^\star} \norm{\yy_{i} - \bar\xx}^2                                                                          \\
                                                                & \hspace{2cm} + (\abs{\tilde\cG \cap \cS^\star} - 2 \abs{\tilde\cB \cap \cS^\star})\norm{\yy_{k^\star}  - \bar\xx}^2\,.
    \end{align*}
    We can rearrange the above equation as
    \begin{align*}
        \norm{\yy_{k^\star}  - \bar\xx}^2 & \leq \frac{1}{(\abs{\tilde\cG \cap \cS^\star} - 2 \abs{\tilde\cB \cap \cS^\star})}(\sum_{i \in \cS^\star}2  \norm{\yy_{k^\star} - \yy_i}^2 + \sum_{i \in \tilde\cG \cap \cS^\star} 2\norm{\yy_{i} - \bar\xx}^2) \\
                                          & \leq \frac{1}{(\abs{\cS^\star} - 3 \abs{\tilde\cB })}(\sum_{i \in \cS^\star}2  \norm{\yy_{k^\star} - \yy_i}^2 + \sum_{i \in \tilde\cG \cap \cS^\star} 2\norm{\yy_{i} - \bar\xx}^2)                              \\
                                          & \leq  \frac{1}{(\abs{\cS^\star} - 3 \abs{\tilde\cB })}(2\min_{k, |\cS|=3m/4}\sum_{i \in \cS}  \norm{\yy_{k} - \yy_i}^2 + \sum_{i \in \tilde\cG } 2\norm{\yy_{i} - \bar\xx}^2).
    \end{align*}
    Taking expectation now on both sides yields
    \[
        \E\norm{\yy_{k^\star}  - \bar\xx}^2 \leq \frac{4m \tilde\rho^2}{\abs{\cS^\star} - 3\abs{\tilde\cB}}\,.
    \]
    Now, recall that we used a bucketing value of $s = \nicefrac{\delta_{\max}}{\delta}$ where for \krum we have $\delta_{\max} = \nicefrac{1}{4} - \nu$. Then, the number of Byzantine workers can be bounded as  $\abs{\tilde\cB} \leq m(1/4 - \nu)$. This gives the final result that
    \[
        \E\norm{\yy_{k^\star}  - \bar\xx}^2 \leq \frac{4m \tilde\rho^2}{3m/4 - 3(m/4 -\nu m)} = \frac{4\tilde\rho^2}{3\nu} \leq \frac{4}{3\nu (1/4 - \nu)} \delta \rho^2\,.
    \]
    Thus, \krum with bucketing indeed satisfies \Cref{definition:robust-agg} with $\delta_{\max} = (\nicefrac{1}{4} - \nu)$ and $c = 4/(3\nu(\nicefrac{1}{4}-\nu))$.

    \paragraph{Robustness of Geometric median.} Geometric median computes the minimum of the following optimization problem
    \[
        \yy^\star = \argmin_{\yy} \sum_{i \in [m]} \norm{\yy - \yy_i}_2\,.
    \]
    We will adapt Lemma~24 of \citet{cohen2016geometric}, which itself is based on \citep{minsker2015geometric}. For a good bucket $i \in \tilde\cG$ and bad bucket $j \in \tilde\cB$:
    \begin{align*}
        \norm{\yy^\star - \yy_i}_2 & \geq \norm{\yy^\star - \bar\xx}_2 - \norm{\yy_i - \bar\xx}_2 \text{ for } i\in\tilde\cG \text{, and } \\
        \norm{\yy^\star - \yy_j}_2 & \geq \norm{\yy_j - \bar\xx}_2 - \norm{\yy^\star - \bar\xx}_2 \text{ for } j\in\tilde\cB\,.
    \end{align*}
    Summing this over all buckets we have
    \begin{align*}
        \sum_{i \in [n]}\norm{\yy^\star - \yy_i} & \geq (\abs{\tilde\cG} - \abs{\tilde\cB})\norm{\yy^\star - \bar\xx}  + \sum_{j \in \tilde\cB}\norm{\yy_j - \bar\xx} - \sum_{i \in \tilde\cG}\norm{\yy_i - \bar\xx}                                 \\
        \Rightarrow \norm{\yy^\star - \bar\xx}   & \leq \frac{1}{(\abs{\tilde\cG} - \abs{\tilde\cB})}\rbr*{ \sum_{i \in [n]}\norm{\yy^\star - \yy_i} - \sum_{j \in \tilde\cB}\norm{\yy_j - \bar\xx} +\sum_{i \in \tilde\cG}\norm{\yy_i - \bar\xx} }  \\
                                                 & = \frac{1}{(\abs{\tilde\cG} - \abs{\tilde\cB})}\rbr*{ \min_{\yy}\sum_{i \in [n]}\norm{\yy - \yy_i} - \sum_{j \in \tilde\cB}\norm{\yy_j - \bar\xx} +\sum_{i \in \tilde\cG}\norm{\yy_i - \bar\xx} } \\
                                                 & \leq \frac{2}{(\abs{\tilde\cG} - \abs{\tilde\cB})}\rbr*{ \sum_{i \in \tilde\cG}\norm{\yy_i - \bar\xx} }\,.\end{align*}
    The last step we substituted $\yy = \bar\xx$. Squaring both sides, expanding, and then taking expectation gives
    \begin{align*}
        \E\norm{\yy^\star - \bar\xx}^2 & \leq \frac{4}{(\abs{\tilde\cG} - \abs{\tilde\cB})^2}\E \rbr*{ \sum_{i \in \tilde\cG}\norm{\yy_i - \bar\xx} }^2             \\
                                       & \leq \frac{4}{(\abs{\tilde\cG} - \abs{\tilde\cB})^2}\rbr*{\abs{\tilde\cG} \sum_{i \in \tilde\cG}\E\norm{\yy_i - \bar\xx}^2 } \\
                                       & \leq \frac{4\abs{\tilde\cG}^2}{(n - 2\abs{\tilde\cB})^2}\tilde\rho^2\,.
    \end{align*}
    Now, recall that we used a bucketing value of $s = \nicefrac{\delta_{\max}}{\delta}$ where for \krum we have $\delta_{\max} = \nicefrac{1}{2} - \nu$. Then, the number of Byzantine workers can be bounded as  $\abs{\tilde\cB} \leq n(1/2 - \nu)$. This gives the final result that
    \[
        \E\norm{\yy^\star - \bar\xx}^2 \leq \frac{4n^2}{4 n^2 \nu^2}\tilde\rho^2 \leq \frac{\tilde\rho^2}{\nu^2} \leq \frac{1}{\nu(1/2 - \nu)} \delta\rho^2\,.
    \]
    Thus, geometric median with bucketing indeed satisfies \Cref{definition:robust-agg} with $\delta_{\max} = (\nicefrac{1}{2} - \nu)$ and $c = 1/(\nu(\nicefrac{1}{2}-\nu))$. Note that geometric median has better theoretical performance than \krum.

    \paragraph{Robustness of Coordinate-wise median.} The proof of coordinate-wise median largely follows that of the geometric median.  First, we note that we can separate out the objective by coordinates
    \[
        \E\norm{\cm(\yy_1, \dots, \yy_m) - \bar\xx}^2 = \sum_{l=1}^d \E\rbr*{\cm([\yy_1]_l, \dots, [\yy_m]_l) - [\bar\xx]_l}^2\,.
    \]
    Then note that, for any fixed coordinate $l \in [d]$ and fixed good worker $i \in \cG$, we have $\E([\yy_i]_l - [\bar\xx]_l)^2 \leq \E \norm{\yy_i - \bar\xx}^2 \leq \tilde\rho^2$. Thus, we can simply analyze coordinate-wise median as $d$ separate (geometric) median problems on scalars. Thus for any fixed coordinate $l \in [d]$, we have
    \[
        \E\rbr*{\cm([\yy_1]_l, \dots, [\yy_m]_l) - [\bar\xx]_l}^2 \leq \frac{\tilde\rho^2}{\nu^2} \Rightarrow \E\norm{\cm(\yy_1, \dots, \yy_m) - \bar\xx}^2 \leq \frac{d\tilde\rho^2}{\nu^2} \leq \frac{d}{\nu(1/2 - \nu)} \delta\rho^2\,.
    \]
    Thus, coordinate-wise median with bucketing indeed satisfies \Cref{definition:robust-agg} with $\delta_{\max} = (\nicefrac{1}{2} - \nu)$ and $c = d/(\nu(\nicefrac{1}{2}-\nu))$.

    \section{Lower bounds on non-iid data (Proof of \Cref{thm:lower-bound})}
    Our proof builds two sets of functions $\{f^1_i(\xx) \, |\, i \in \cG^1 \}$ and $\{f^2_i(\xx) \, |\, i \in \cG^2 \}$ and will show that in the presence of $\delta$-fraction of Byzantine workers, no algorithm can distinguish between them. Since the problems have different optima, this means that the algorithm necessarily has an error on at least one of them.

    For the first set of functions, let there be \emph{no} bad workers and hence $\cG^1 = [n]$. Then, we define the following functions for any $i \in [n]$:
    \[
        f^1_i(x) = \begin{cases}
            \frac{\mu}{2}x^2 - \zeta\delta^{-1/2} x & \text{ for } i \in \{1,\dots, \delta n\}         \\
            \frac{\mu}{2}x^2                            & \text{ for } i \in \{\delta n + 1, \dots, n\}\,.
        \end{cases}
    \]
    Defining $G := \zeta \delta^{1/2}$, the average function which is our objective is
    \[
        f^1(x) = \frac{1}{n}\sum_{i=1}^n f^1_i(x) = \frac{\mu}{2}x^2 - G x \,.
    \]
    The optimum of our $f^1(x)$ is at $x = \frac{G}{\mu}$.
    Note that the gradient heterogeneity amongst these workers is bounded as
    \begin{align*}
        \E_{i \sim [n]}\norm{\nabla f^1_i(x) - \nabla f^1(x)}^2
        =&
        \delta (\zeta\delta^{-1/2} - \zeta\delta^{1/2})^2
        +(1-\delta)(\zeta\delta^{1/2})^2\\
        =& \zeta^2 (1-\delta)^2 + \zeta^2(1-\delta)\delta
        = \zeta^2 (1-\delta)
        \le \zeta^2.
    \end{align*}

    Now, we define the second set of functions. Here, suppose that we have $\delta n$ Byzantine attackers with $\cB^2 = \{1, \dots, \delta n\}$. Then, the functions of the good workers are defined as
    \[
        f^2_i(x) = \frac{\mu}{2}x^2  \text{ for } i \in \cG^2 =\{\delta n + 1, \dots, n\}\,.
    \]
    We then have that the second average objective is
    \[
        f^2(x) = \frac{1}{\abs{\cG^2}}\sum_{i \in \cG^2} f^2_i(x) = \frac{\mu}{2}x^2  \,.
    \]
    Here, we have gradient heterogeneity of 0 and hence is smaller than $\zeta^2$. The optimum of $f^2(x)$ is at $x=0$.
    The Byzantine attackers simply imitate as if they have the functions:
    \[
        f^2_j(x) = \frac{\mu}{2}x^2 - \zeta\delta^{-1/2} x \text{ for } j \in \cB^2 = \{1,\dots, \delta n\}\,.
    \]

    Note that the set of functions,  $\{f^1_1(\xx), \dots, f^1_n(x)\}$ is exactly identical to the set $\{f^2_1(\xx), \dots, f^2_n(x)\}$. Only the identity of the good workers $\cG^1$ and $\cG^2$ are different, leading to different objective functions $f^1(x)$ and $f^2(x)$. However, since the algorithm does not have access to $\cG$, its output on each of them is identical i.e.
    \[
        x^{\text{out}} = \alg(f^1_1(\xx), \dots, f^1_n(x)) = \alg(f^2_1(\xx), \dots, f^2_n(x))\,.
    \]
    However, the leads to making a large error in least one of $f^1$ and $f^2$ since they have different optimum. This proves a lower bound error of
    \[
        \max_{k \in \{1,2\}} f^k(x^{\text{out}}) - f^k(x^\star) \geq \mu\rbr*{\frac{G}{2\mu}}^2 = \frac{\delta \zeta^2}{4\mu}\,.
    \]
    Similarly, we can also bound the gradient norm error bound as
    \[
        \max_{k \in \{1,2\}} \norm{\nabla f^k(x^{\text{out}})}^2 \geq \mu^2\rbr*{\frac{G}{2\mu}}^2 = \frac{\delta \zeta^2}{4}\,.
    \]

    \qed

    \section{Convergence of robust optimization on non-iid data (Theorems~\ref{thm:convergence-general} and \ref{thm:overparam-convergence})}

    We will prove a more general convergence theorem which generalizes Theorems~\ref{thm:convergence-general} and \ref{thm:overparam-convergence}.
    \begin{theorem}\label{thm:appendix-convergence}
        Suppose we are given a $(\delta_{\max}, c)$-\ragg satisfying \Cref{definition:robust-agg}, and $n$ workers of which a subset $\cG$ of size at least $\abs{\cG} \geq n(1 - \delta)$ faithfully follow the algorithm for $\delta \leq \delta_{\max}$. Further, for any good worker $i \in \cG$ let $f_i$ be a possibly non-convex function with $L$-Lipschitz gradients, and the stochastic gradients on each worker be independent, unbiased and satisfy
        \[
            \E_{\xiv_i}\norm{\gg_i(\xx) - \nabla f_i(\xx)}^2 \leq \sigma^2 \text{ and }  \E_{j \sim \cG}\norm{\nabla f_j(\xx) - \nabla f(\xx)}^2 \leq \zeta^2 + B^2 \norm{\nabla f(\xx)}^2 \,, \quad \forall \xx\,,
        \]
        where $\delta \leq 1/(60cB^2)$. Then, for $F^0 := f(\xx^0) - f^\star$, the output of \Cref{algo:bucketing_sgd} with step-size $\eta =  \min\rbr*{\cO\rbr[\bigg]{\sqrt{\frac{L F^0 + c\delta(\zeta^2 + \sigma^2) }{T  L^2 \sigma^2 (n^{-1} + c\delta)}}}, \frac{1}{8L}} $ and momentum parameter $\beta = (1 - 8L\eta)$ satisfies \vspace{-2mm}
        \begin{align*}
            \frac{1}{T}\sum_{t=1}^T \E \norm{\nabla f(\xx^{t-1})}^2 & \leq
            \cO \rbr[\Big]{ \frac{1}{1 \!-\! 60c\delta B^2} \cdot \rbr[\Big]{c \delta \zeta^2 +
                    \sigma\sqrt{\frac{LF^0}{T} (c\delta + \nicefrac{1}{n})} + \frac{L F^0}{T} } }\,.
        \end{align*}
    \end{theorem}
    \paragraph{Notes on $\delta \leq 1/(60cB^2)$.} In practice the upper bound $\delta \leq 1/(60cB^2)$ does not put an extra strict constraint on $\delta$.
    This is because one can always decrease $B^2$ and increase $\zeta^2$ such that $\E_{j \sim \cG}\norm{\nabla f_j(\xx) - \nabla f(\xx)}^2 \leq \zeta^2 + B^2 \norm{\nabla f(\xx)}^2$ holds for a sufficiently large domain of $\xx$.

    \paragraph{Definitions.}

    Recall our algorithm which performs for $t \geq 2$ the following update with $(1 - \beta) = \alpha$:
    \begin{align*}
        \mm_i^t & = (1-\alpha) \mm_i^{t-1} + \alpha \gg_i(\xx^{t-1}) \quad \text{ for every } i \in \cG\,, \\
        \xx^t   & = \xx^{t-1} - \eta \ragg(\mm_1^t, \dots, \mm_n^t)\,.
    \end{align*}
    For $t=1$, we use $\alpha=0$ i.e. $\mm_i^1 = \gg_i(\xx^{0})$. Let us also define the actual and ideal momentum aggregates as
    \[
        \mm^t := \ragg(\mm_1^t, \dots, \mm_n^t) \quad \text{and}\quad \bar\mm^t := \frac{1}{\abs{\cG}}\sum_{i\in \cG}\mm^t_i\,.
    \]

    We state several supporting lemmas before proving our main Theorem~\ref{thm:appendix-convergence}. We will loosely follow the proof of Byzantine robustness in the iid case by \citet{karimireddy2020learning}, with the key difference of Lemma~\ref{lem:sgdm-byz-agg-err} which accounts for the non-iid error.

    \begin{lemma}[Aggregation error]\label{lem:sgdm-byz-agg-err}
        Given that $\ragg$ satisfies \Cref{definition:robust-agg} holds, the error between the ideal average momentum $\bar\mm^t$ and the output of the robust aggregation rule $\mm^t$ for any $t\geq 2$ can be bounded as
        \[
            \E\norm{\mm^t - \bar\mm^t}^2 \leq c\delta \rho_t^2\,,
        \]
        where we define for $t\geq 2$
        \[
            \rho_t^2:= 4(6\alpha \sigma^2 + 3\zeta^2) + 4(6\sigma^2 - 3\zeta^2)(1-\alpha)^t + 12\sum_{k=1}^t (1-\alpha)^{t-k} \alpha B^2 \E\norm{\nabla f(\xx^{k-1})}^2\,.
        \]
        For $t=1$ we can simplify the bound as $\rho_1^2 :=  24c\delta \sigma^2 + 12c \delta \zeta^2+ 12c\delta B^2 \norm{\nabla f(\xx^{0})}^2$.
    \end{lemma}
    \begin{proof}
        Let $\E_{\xi^t}:=\E_{\xi_1^t,\ldots,\xi_n^t,\xi_1^{t-1},\ldots,\xi_n^{t-1},\ldots,\xi_1^{0},\ldots,\xi_n^{0}}$ be the expectation with respect to all of the randomness until time $t$ and let $\E_i:=\E_{i\in\cG}$ and $\E=\E_{\xi^t}\E_i$.
        Expanding the definition of the worker momentum for a fixed good workers $i\in \cG$,
        \begin{align*}
            \E_{\xi^t}\norm{\mm_i^t - \E_{\xi^t}[\mm_i^t]}^2 & = \E_{\xi^t}\norm{\alpha (\gg_i(\xx^{t-1}) - \nabla f_i(\xx^{t-1})) + (1 - \alpha)(\mm_i^{t-1} - \E_{\xi^t}[\mm_i^{t-1}])}^2 \\
                                                    & \leq \E_{\xi^{t-1}}\norm{(1 - \alpha)(\mm_i^{t-1} - \E[\mm_i^{t-1}])}^2 + \alpha^2 \sigma^2                              \\
                                                    & \leq (1 - \alpha)\E_{\xi^{t-1}}\norm{\mm_i^{t-1} -  \E[\mm_i^{t-1}]}^2 + \alpha^2 \sigma^2\,.
        \end{align*}
        Unrolling the recursion above yields
        \begin{align*}
            \E_{\xi^t}\norm{\mm_i^t - \E_{\xi^t}[\mm_i^t]}^2 \leq \rbr*{\sum_{\ell=2}^t(1 - \alpha)^{t-\ell}} \alpha^2\sigma^2 + (1 - \alpha)^{t-1}\sigma^2 \leq \sigma^2(\alpha + (1 - \alpha)^{t-1})\,.
        \end{align*}
        Similar computation also shows
        \[
            \E_{\xi^t}\norm{\bar\mm^t - \E_{\xi^t}[\bar\mm^t]}^2 \leq \frac{\sigma^2}{n} (\alpha + (1 - \alpha)^{t-1})\,.
        \]
        So far, the expectation was only over the stochasticity of the gradients of worker $i$. Note that we have $\E_{\xi^t}[\mm_i^t ] = \E_{\xi^{t-1}}[\alpha\nabla f_i(\xx^{t-1}) + (1-\alpha) \mm_i^{t-1}]$. Now, suppose we sample $i$ uniformly at random from $\cG$. Then,
        \begin{align*}
            &\E_{i}\norm*{\E_{\xi^t}[\mm_i^t ] - \E_{\xi^t}[\bar\mm^t]}^2\\
            \!=\!& \E_i\norm{\alpha \E_{\xi^{t-1}}[\nabla f_i(\xx^{t-1}) - \nabla f(\xx^{t-1})] + (1 - \alpha)(\E_{\xi^{t-1}}[\mm_i^{t-1} ] - \E_{\xi^{t-1}}[\bar\mm^{t-1}])}^2
            \\
            \! \leq &(1- \alpha)\E_i\norm{\E_{\xi^{t-1}}[\mm_i^{t-1} ] \!-\! \E_{\xi^{t-1}}[\bar\mm^{t-1}]}^2 \!+\! \alpha\E_i\norm{\E_{\xi^{t-1}} [\nabla f_i(\xx^{t-1}) \!-\! \nabla f(\xx^{t-1})]}^2 \\
            \! \leq &(1- \alpha)\E_i\norm{\E_{\xi^{t-1}}[\mm_i^{t-1} ] \!-\! \E_{\xi^{t-1}}[\bar\mm^{t-1}]}^2 \!+\! \alpha\E_i \E_{\xi^{t-1}} \norm{\nabla f_i(\xx^{t-1}) \!-\! \nabla f(\xx^{t-1})}^2 \\
            \! \leq\!& (1- \alpha)\E_i\norm{\E_{\xi^{t-1}}[\mm_i^{t-1} ] - \E_{\xi^{t-1}} [\bar\mm^{t-1}]}^2 + \alpha\zeta^2 + \alpha B^2\E\norm{\nabla f(\xx^{t-1})}^2
        \end{align*}
        where the second inequality uses the probabilistic Jensen's inequality.
        Note that here we only get $\alpha$ instead of $\alpha^2$ as before. This is because the randomness in the sampling $i$ of $\nabla f_i(\xx^{t-1})$ is not independent of the second term $\E_{\xi^{t-1}} [\mm_i^{t-1} ] - \E_{\xi^{t-1}} [\bar\mm^{t-1}]$. Expanding this we get,
        \[
            \E_{i}\norm*{\E_{\xi^t}[\mm_i^t ] - \E_{\xi^t}[\bar\mm^t]}^2 \leq \zeta^2(1 - (1 - \alpha)^{t}) + \sum_{k=1}^t (1-\alpha)^{t-k} \alpha B^2 \E\norm{\nabla f(\xx^{k-1})}^2\,.
        \]

        We can combine all three bounds above as
        \begin{align*}
            \E\norm{\mm_i^t - \bar\mm^t}^2 & \leq 3 \E\norm{\mm_i^t - \E_{\xi^t}[\mm_i^t]}^2 + 3 \E\norm{\bar\mm^t - \E_{\xi^t}[\bar\mm^t]}^2 + 3\E_{i}\norm{\E_{\xi^t}[\mm_i^t] - \E_{\xi^t}[\bar\mm^t]}^2          \\
             & = 3 \E_i\E_{\xi^t}\norm{\mm_i^t - \E_{\xi^t}[\mm_i^t]}^2 + 3 \E_{\xi^t}\norm{\bar\mm^t - \E_{\xi^t}[\bar\mm^t]}^2 + 3\E_{i}\norm{\E_{\xi^t}[\mm_i^t] - \E_{\xi^t}[\bar\mm^t]}^2          \\                                               & \leq (6\alpha \sigma^2 + 3\zeta^2) + (6\sigma^2 - 3\zeta^2)(1-\alpha)^t + 3\sum_{k=1}^t (1-\alpha)^{t-k} \alpha B^2 \E\norm{\nabla f(\xx^{k-1})}^2\,.
        \end{align*}
        Therefore for $i, j\in\cG$
        \begin{align*}
            \E\norm{\mm_i^t - \mm_j^t}^2 & \le 2\E\norm{\mm_i^t - \bar\mm^t}^2 + 2\E\norm{\mm_j^t - \bar\mm^t}^2         \\
            & \leq 4(6\alpha \sigma^2 + 3\zeta^2) + 4(6\sigma^2 - 3\zeta^2)(1-\alpha)^t + 12\sum_{k=1}^t (1-\alpha)^{t-k} \alpha B^2 \E\norm{\nabla f(\xx^{k-1})}^2\,.
        \end{align*}        
        Recall that the right hand side was defined to be $\rho_t^2$. Using \Cref*{definition:robust-agg}, we can show that the output of the aggregation rule $\ragg$ satisfies the condition in the lemma.
    \end{proof}

    One major caveat in the above lemma is that here $\rho^2$ cannot be known to the robust aggregation since it involves $\E\norm{\nabla f(\xx^{k-1})}^2$ whose value we do not have access to. However, this does not present a hurdle to \emph{agnostic} aggregation rules which are automatically adaptive to the value of $\rho^2$. Deriving a similarly provable adaptive clipping method is a very important open problem.

    \begin{lemma}[Descent bound]\label{lem:sgdm-byz-descent}
        For any $\alpha \in [0,1]$ for $t\geq2$, $\eta \leq \frac{1}{L}$, and an $L$-smooth function $f$ we have for any $t\geq 1$
        \[
            \E[f(\xx^t)] \leq f(\xx^{t-1}) - \frac{\eta}{2}\norm{\nabla f(\xx^{t-1})}^2 + \eta\E\norm{\bar\ee^t}^2 + \eta\E\norm{\mm^t - \bar\mm^t}^2\,.
        \]
        where $\bar\ee^t:=\bar\mm^t - \nabla f(\xx^{t-1})$.
    \end{lemma}
    \begin{proof}
        By the smoothness of the function $f$ and the server update,
        \begin{align*}
            f(\xx^t) & \leq f(\xx^{t-1}) - \eta \inp{\nabla f(\xx^{t-1})}{\mm^t} + \frac{L \eta^2}{2} \norm{\mm^t}^2                                   \\
                     & \leq f(\xx^{t-1}) - \eta \inp{\nabla f(\xx^{t-1})}{\mm^t} + \frac{\eta}{2} \norm{\mm^t}^2                                       \\
                     & = f(\xx^{t-1}) + \frac{\eta}{2} \norm{\mm^t - \nabla f(\xx^{t-1})}^2 - \frac{\eta}{2}\norm{\nabla f(\xx^{t-1})}^2               \\
                     & = f(\xx^{t-1}) + \frac{\eta}{2} \norm{\mm^t \pm \bar\mm^t - \nabla f(\xx^{t-1})}^2 - \frac{\eta}{2}\norm{\nabla f(\xx^{t-1})}^2 \\
                     & \leq f(\xx^{t-1}) + \eta\norm{\bar\ee^t}^2 + \eta\norm{\mm^t - \bar\mm^t}^2 - \frac{\eta}{2}\norm{\nabla f(\xx^{t-1})}^2\,.
        \end{align*}
        Here we used the identities that $-2ab = (a - b)^2 - a^2 - b^2$, and Young's inequality that $(a+b)^2 \leq (1+\gamma)a^2 + (1 + \tfrac{1}{\gamma})b^2$ for any positive constant $\gamma \geq 0$ (here we used $\gamma=1$). Taking conditional expectation on both sides yields the lemma.
    \end{proof}

    \begin{lemma}[Error bound]\label{lem:sgdm-byz-error}
        Using any constant momentum and step-sizes such that $1 \geq \alpha \geq 8 L \eta$ for $t \geq 2$, we have  for an $L$-smooth function $f$ that $\E \norm{\bar\ee^{1}}^2 \leq \tfrac{2\sigma^2}{n}$ and for $t \geq 2$
        \begin{align*}
            \E \norm{\bar\ee^{t}}^2 & \leq  (1 - \tfrac{2\alpha}{5})\E\norm{\bar\ee^{t-1}}^2  + \tfrac{\alpha}{10}\E\norm{\nabla f(\xx^{t-2}) }^2  + \tfrac{\alpha}{10}\E\norm{\mm^{t-1} - \bar\mm^{t-1}}^2 + \alpha^2 \tfrac{2\sigma^2}{n}\,.
        \end{align*}
    \end{lemma}
    \begin{proof}
        Let us define $\bar\gg(\xx) := \frac{1}{\abs{\cG}}\sum_{i \in \cG} \gg_i(\xx)$. This implies that
        \[
            \E\norm{\bar\gg(\xx) - \nabla f(\xx)}^2 \leq \frac{\sigma^2}{\abs{\cG}} \leq \frac{2\sigma^2}{n}\,.
        \]

        Then by definition of $\bar\mm$, we can expand the error as:
        \begin{align*}
            \E \norm{\bar\ee^{t}}^2 & = \E \norm{\bar\mm^t - \nabla f(\xx^{t-1})}^2                                                                                                 \\
                                    & = \E \norm{\alpha \bar\gg(\xx^{t-1}) + (1-\alpha)\bar\mm^{t-1} - \nabla f(\xx^{t-1})}^2                                                       \\
                                    & \leq \E \norm{\alpha \nabla f(\xx^{t-1}) + (1-\alpha)\bar\mm^{t-1} - \nabla f(\xx^{t-1})}^2 + \frac{2\alpha^2\sigma^2}{n}                     \\
                                    & = (1 - \alpha)^2\E \norm{(\bar\mm^{t-1} - \nabla f(\xx^{t-2})) + (\nabla f(\xx^{t-2}) - \nabla f(\xx^{t-1}))}^2 + \frac{2\alpha^2\sigma^2}{n} \\
                                    & \leq (1 - \alpha)(1 + \tfrac{\alpha}{2})\E \norm{(\bar\mm^{t-1} - \nabla f(\xx^{t-2}))}^2
            \\&\hspace{2cm}+ (1 - \alpha)(1 + \tfrac{2}{\alpha})\E\norm{\nabla f(\xx^{t-2}) - \nabla f(\xx^{t-1})}^2 + \frac{2\alpha^2\sigma^2}{n} \\
                                    & \leq (1 - \tfrac{\alpha}{2})\E\norm{\bar\ee^{t-1}}^2 + \tfrac{2 L^2}{\alpha}\E\norm{\xx^{t-2} - \xx^{t-1}}^2 + \frac{2\alpha^2\sigma^2}{n}    \\
                                    & = (1 - \tfrac{\alpha}{2})\E\norm{\bar\ee^{t-1}}^2 + \tfrac{2 L^2 \eta^2}{\alpha}\E\norm{\mm^{t-1}}^2 + \frac{2\alpha^2\sigma^2}{n}            \\
                                    & \leq (1 - \tfrac{\alpha}{2})\E\norm{\bar\ee^{t-1}}^2 + \tfrac{6 L^2 \eta^2}{\alpha}\norm{\bar\ee^{t-1}}^2
            \\&\hspace{2cm}+ \tfrac{6 L^2 \eta^2}{\alpha}\E\norm{ \mm^{t-1} - \bar\mm^{t-1}}^2 + \tfrac{6 L^2 \eta^2}{\alpha}\E\norm{\nabla f(\xx^{t-2}) }^2 + \frac{2\alpha^2\sigma^2}{n}\,.
        \end{align*}
        Our choice of the momentum parameter $\alpha$ implies $64L^2\eta^2 \leq \alpha^2$ and yields the lemma statement.

    \end{proof}

    \paragraph{Proof of Theorem~\ref{thm:appendix-convergence}.}
    Scale the error bound Lemma~\ref{lem:sgdm-byz-error} by $\frac{5\eta}{2 \alpha}$ and add it to the descent bound Lemma~\ref{lem:sgdm-byz-descent} taking expectations on both sides to get for $t \geq 2$
    \begin{align*}
        \E[f(\xx^t)] + \tfrac{5\eta}{2\alpha}\E \norm{\bar\ee^{t}}^2 & \leq \E[f(\xx^{t-1})] - \tfrac{\eta}{2}\E\norm{\nabla f(\xx^{t-1})}^2 + \eta\E\norm{\bar\ee^t}^2 + \eta\E\norm{\mm^t - \bar\mm^t}^2 +
        \\&\hspace{2cm} \tfrac{5\eta}{2\alpha}\E\norm{\bar\ee^{t-1}}^2 - \eta\E\norm{\bar\ee^{t-1}}^2  + \tfrac{\eta}{4}\E\norm{\nabla f(\xx^{t-2}) }^2
        \\&\hspace{2cm} + \tfrac{\eta}{4}\E\norm{\mm^{t-1} - \bar\mm^{t-1}}^2 +  5\eta\alpha\frac{\sigma^2}{n}.
    \end{align*}
    Now, let use the aggregation error Lemma~\ref{lem:sgdm-byz-agg-err} to bound $\E\norm{\mm^{t-1} - \bar\mm^{t-1}}^2$ and $\E\norm{\mm^{t} - \bar\mm^{t}}^2$ in the above expression to get
    \begin{align*}
        \E[f(\xx^t)] + \tfrac{5\eta}{2\alpha}\E \norm{\bar\ee^{t}}^2 & \leq \E[f(\xx^{t-1})] - \tfrac{\eta}{2}\E\norm{\nabla f(\xx^{t-1})}^2 + \eta\E\norm{\bar\ee^t}^2
        \\&\hspace{1.5cm}+ \tfrac{5\eta}{2\alpha}\E\norm{\bar\ee^{t-1}}^2 - \eta\E\norm{\bar\ee^{t-1}}^2  + \tfrac{\eta}{4}\E\norm{\nabla f(\xx^{t-2}) }^2+ 5\eta\alpha\tfrac{\sigma^2}{n}
        \\&\hspace{1.5cm} + {5\eta c\delta}((6\alpha \sigma^2 + 3\zeta^2) + 6\sigma^2 (1-\alpha)^{t-2})
        \\&\hspace{1.5cm} + \eta c\delta \bigg(3\sum_{k=1}^{t-1} (1-\alpha)^{t-1-k} \alpha B^2 \E\norm{\nabla f(\xx^{k-1})}^2\bigg)
        \\&\hspace{1.5cm} + 4\eta c\delta \bigg(3\sum_{k=1}^{t} (1-\alpha)^{t-k} \alpha B^2 \E\norm{\nabla f(\xx^{k-1})}^2\bigg).
    \end{align*}
    Let us define $S_t := \sum_{k=1}^{t} (1-\alpha)^{t-k} \alpha B^2 \E\norm{\nabla f(\xx^{k-1})}^2$. Then, $S_t$ satisfies the recursion:
    \[
        \tfrac{1}{\alpha}S_t = (\tfrac{1}{\alpha} -1) S_{t-1} + B^2\E\norm{\nabla f(\xx^{t-1})}^2 \,.
    \]
    Adding $ \frac{3 \eta c \delta (\frac{5}{\alpha} - 4) }{\alpha}S_t$ on both sides to the bound above and rearranging gives the following for $t \geq 2$

    \begin{align*}
         & \underbrace{\E~f(\xx^t) - f^\star + (\tfrac{5 \eta}{2\alpha} - \eta)\E \norm{\bar\ee^{t}}^2 + \frac{\eta}{4}\E\norm{\nabla f(\xx^{t-1})}^2 + \frac{3 \eta c \delta (\frac{5}{\alpha} - 4)}{\alpha}S_t}_{=: \cE_t}                                   \\
         & \hspace{1.5cm} \leq \underbrace{\E~f(\xx^{t-1}) - f^\star + (\tfrac{5 \eta}{2\alpha} - \eta)\E \norm{\bar\ee^{t-1}}^2 + \frac{\eta}{4}\E\norm{\nabla f(\xx^{t-2})}^2 + \frac{3 \eta c \delta (\frac{5}{\alpha} - 4)}{\alpha}S_{t-1}}_{=: \cE_{t-1}}
        \\&\hspace{3cm} (- \tfrac{\eta}{4} + { 15 \eta c \delta B^2})\E \norm{\nabla f(\xx^{t-1})}^2
        \\&\hspace{3cm} + \frac{5\eta\alpha}{n}\sigma^2 + {5\eta c\delta}\rbr*{(6\alpha \sigma^2 + 3\zeta^2) + 6\sigma^2(1-\alpha)^{t-2}}\\
         & \hspace{1.5cm}  \leq \cE_{t-1} - \frac{\eta}{4}(1-60c\delta B^2)\E \norm{\nabla f(\xx^{t-1})}^2
        \\&\hspace{3cm} + \underbrace{5\eta\alpha \sigma^2 \rbr*{\tfrac{1}{n} + 6c\delta(1 + \tfrac{1}{\alpha}(1-\alpha)^{t-2})} + 15\eta c \delta \zeta^2 }_{=: \eta\xi_{t-1}^2}\,.
    \end{align*}
    Further, specializing the descent bound Lemma~\ref{lem:sgdm-byz-descent} and error bound Lemma~\ref{lem:sgdm-byz-error} for $t=1$ we have
    \begin{align*}
        \cE_1 & =  \E~f(\xx^{1}) - f^\star + \frac{3 \eta}{2}\E \norm{\bar\ee^{1}}^2 + \frac{\eta}{4}\E\norm{\nabla f(\xx^{0})}^2   + 3 \eta c \delta B^2 (\frac{5}{\alpha} - 4)\norm{\nabla f(\xx^0)}^2                 \\
              & \leq f(\xx^{0}) - f^\star + \frac{5\eta}{2}\E \norm{\bar\ee^{1}}^2 - \frac{\eta}{4}(1-60c\delta B^2)\E\norm{\nabla f(\xx^{0})}^2 + \eta \E\norm{\mm_1 - \bar\mm_1}^2                         \\
              & \leq f(\xx^{0}) - f^\star - \frac{\eta}{4}(1-60c\delta B^2)\E\norm{\nabla f(\xx^{0})}^2 + \frac{5\eta\sigma^2}{n} +   12c \delta \eta( 2\sigma^2 + \zeta^2 +  B^2 \norm{\nabla f(\xx^{0})}^2) \\
              & = f(\xx^{0}) - f^\star - \frac{\eta}{4}(1-60c\delta B^2)\E\norm{\nabla f(\xx^{0})}^2 + \eta\xi_0^2\,.
    \end{align*}
    Above, we defined $\xi_0^2 := \frac{5\sigma^2}{n} + 12c \delta ( 2\sigma^2 + \zeta^2 +  B^2 \norm{\nabla f(\xx^{0})}^2)$.
    Summing over $t$ from $2$ until $T$, again rearranging our recursion for $\cE_t$, and adding $(1 - 3c \delta B^2)\E \norm{\nabla f(\xx^{0})}^2 $ on both sides gives
    \begin{align*}
        (1 - 60c \delta B^2)\frac{1}{T}\sum_{t=1}^T\E \norm{\nabla f(\xx^{t-1})}^2
         & \leq \frac{4(f(\xx^{0}) - f^\star)}{\eta T} + \frac{1}{T}\sum_{t=1}^T 4\xi_{t-1}^2                                                                             \\
         & = \frac{4(f(\xx^{0}) - f^\star)}{\eta T} + \frac{4 \xi_0^2}{T}
        \\&\hspace{1.5cm} + \frac{1}{T}\sum_{t=2}^T 20 \alpha \sigma^2 \rbr*{\tfrac{1}{n} + 6c\delta(1 + \tfrac{1}{\alpha}(1-\alpha)^{t-2})}
        \\&\hspace{1.5cm}+ \frac{1}{T}\sum_{t=2}^T 60 c \delta \zeta^2 \\
         & \leq  \frac{4(f(\xx^{0}) - f^\star)}{\eta T} + \frac{4 \xi_0^2}{T} + 60c\delta \zeta^2 
        \\&\hspace{1.5cm} + 20 \alpha \sigma^2 \rbr*{\tfrac{1}{n} + 6c\delta} + \frac{120 c \delta \sigma^2}{\alpha T}\\
         & =  \frac{4(f(\xx^{0}) - f^\star)}{\eta T} + \frac{120c\delta \sigma^2}{\eta 8 L T} + \eta 160 L \sigma^2 \rbr*{\tfrac{1}{n} + 6c\delta}
        \\&\hspace{1.5cm} + \frac{4 \xi_0^2}{T} + 60c\delta \zeta^2.
    \end{align*}
    The last equality substituted the value of $\alpha = 8L\eta$. Next, let us use the appropriate step-size of
\[
    \eta = \min\rbr*{\sqrt{\frac{4(f(\xx_0) - f^\star) + \tfrac{15c\delta }{L}(\zeta^2 + 2\sigma^2) }{T (160 L \sigma^2 )\rbr*{\tfrac{1}{n} + 6c\delta}}}, \frac{1}{8L}}\,.
\]
This gives the following final rate of convergence:
\begin{align*}
    \frac{1}{T}\sum_{t=1}^T\E \norm{\nabla f(\xx^{t-1})}^2 &                                                                                                                                                                                                                                        \\
                                                           & \hspace{-2cm}\leq \frac{1}{1 - 60c \delta B^2} \cdot \bigg(60c\delta\zeta^2 + \sqrt{\frac{160 L \sigma^2\rbr*{\tfrac{1}{n} + 6c\delta}}{T}}\cdot \sqrt{4(f(\xx_0) - f^\star) + \tfrac{15c\delta }{ L}(\zeta^2 + 2\sigma^2) } \\
                                                           & \hspace{1cm} + \frac{32L(f(\xx^{0}) - f^\star)}{ T} + \frac{15c\delta  \sigma^2}{T}                                                                                                                                          \\
                                                           & \hspace{1cm} + \frac{ \frac{20\sigma^2}{n} +   12c \delta ( 2\sigma^2 + \zeta^2 +  B^2 \norm{\nabla f(\xx^{0})}^2)}{T}\bigg).
\end{align*}
\qed

\section{Updates with respect to reviews.}
\subsection{Additional Related work} \label{ssec:additional}
In this section, we add comments on works which are very close to this paper.

\begin{itemize}
    \item \citet{li2019rsa} propose \textsc{RSA} for Byzantine-resilient distributed learning on heterogeneous data. They introduce an additional $\ell_p$-norm regularized term to 
    the objective to penalize the difference between local iterates and server iterate and show convergence of RSA for strongly convex local objectives and penalized term. However, \textsc{RSA} cannot defend the state-of-the-art attacks like \citep{baruch2019little,xie2019fall} as they didn't utilize the temporal information. 
    Compared to \textsc{RSA}, our method does not assume strongly convexity and we consider more general cost functions with no explicit regularized term. In addition, our method is shown to defend the state-of-the-art attacks.

    \item \citep{pmlr-v139-yang21e} is a parallel work which uses buffer for asynchronous Byzantine-resilient training (BASGD). The buffer and bucketing are similar techniques with vastly different motivations. The key difference between buffer and bucketing is that buffer is only reassigned when timer exceeds a threshold while bucketing reshuffles in each iteration. Therefore, buffer does not guarantee that partial aggregated gradients are identically distributed while bucketing does. In addition, our theoretical analysis does not require bounded gradient assumption $\norm{\nabla f(\xx)}\le D$ for all $\xx$. 
    
    \item \citet{wu2020federated} uses ByrdSAGA for Byzantine-resilient SAGA approach for distributed learning. The key differences between ByrdSAGA and our work are as follows.
    In our setting, there are two sources of variances of the gradients - intra-worker variance $\sigma^2$ and inter-client variance $\zeta^2$. We show that simply using worker momentum suffices to tackle the former and handling the latter $\zeta^2$ is the main challenge. ByrdSAGA assumes that each worker only has finite data points as opposed to the stochastic setting we consider. Hence they can use SAGA on the worker in place of worker momentum to reduce the intra-client variance $\sigma^2$. The effect of $\zeta^2$ (which is our main focus) remains unaffected. 

    Further, they consider the strongly convex setting whereas we analyze non-convex functions. Ignoring $\mu$ for sake of comparison, their Theorem 1 proves convergence to a radius of $\Delta_1 = O(\zeta^2)$ since always $C_\alpha \geq 2$. Thus, their rates are similar to \citep{acharya2021robust} and do not converge to the optimum even when $\delta = 0$. In contrast, our \Cref{thm:convergence-general} proves convergence to a radius of $O(\delta \zeta^2)$. We believe our improved handling of $\zeta^2$ can be combined with their usage of SAGA/variance reduction to yield even faster rates. This we leave for future work.
    
\end{itemize}

\end{document}